\documentclass[accepted]{uai2026} 
                        
\usepackage[american]{babel}

\usepackage[round]{natbib} 
\usepackage{mathtools} 
\usepackage{amsthm}
\usepackage{amssymb}
\usepackage{dblfloatfix}
\newtheorem{theorem}{Theorem}[section]
\newtheorem{proposition}{Proposition}[section]
\newtheorem{lemma}{Lemma}[section]

\newtheorem{definition}{Definition}[section]

\usepackage{booktabs} 
\usepackage{tikz} 

\usepackage{cleveref}
\usepackage{todonotes}
\usepackage{subcaption}
\usepackage{comment}
\newcommand{\R}{\mathbb{R}}
\newcommand{\N}{\mathbb{N}}
\newcommand{\Z}{\mathbb{Z}}

\newcommand{\vx}{\mathbf{x}}

\newcommand{\vd}{\mathrm{d}}

\usepackage{bm} 

\newcommand{\compose}[1]{\overset{\circ}{#1}}

\title{Beyond Mixtures and Products for Ensemble Aggregation: \\ A Likelihood Perspective on Generalized Means}

\author[1]{\href{mailto:<raphael.razafindralambo@inria.fr>?Subject=Your paper}{Raphaël~Razafindralambo}{}}
\author[1]{\href{mailto:<remy.sun@inria.fr>?Subject=Your paper}{Rémy~Sun}}
\author[1]{\href{mailto:<frederic.precioso@univ-cotedazur.fr>?Subject=Your paper}{Frédéric~Precioso}}
\author[2]{\href{mailto:<damien.garreau@uni-wuerzburg.de>?Subject=Your paper}{Damien~Garreau}}
\author[1]{\href{mailto:<pierre-alexandre.mattei@inria.fr>?Subject=Your paper}{Pierre-Alexandre~Mattei}}
\affil[1]{%
    Université Côte d’Azur\\Inria\\CNRS\\I3S/LJAD\\Maasai\\Nice\\France
}
\affil[2]{%
    Julius-Maximilians-Universität Würzburg\\Institute for Computer Science / CAIDAS\\Würzburg, Germany
}

\begin{document}
\maketitle

\begin{abstract}

Density aggregation is a central problem in machine learning, for instance when combining predictions from a Deep Ensemble. The choice of aggregation remains an open question with two commonly proposed approaches being linear pooling (probability averaging) and geometric pooling (logit averaging). In this work, we address this question by studying the normalized generalized mean of order $r \in \R \cup \{-\infty,+\infty\}$ through the lens of log-likelihood, the standard evaluation criterion in machine learning. This provides a unifying aggregation formalism and shows different optimal configurations for different situations. We show that the regime $r \in [0,1]$ is the only range ensuring systematic improvements relative to individual distributions, thereby providing a principled justification for the reliability and widespread practical use of linear ($r=1$) and geometric ($r=0$) pooling. In contrast, we show that aggregation rules with $r \notin [0,1]$ may fail to provide consistent gains with explicit counterexamples. Finally, we corroborate our theoretical findings with empirical evaluations using Deep Ensembles on image and text classification benchmarks.
\end{abstract}

\section{Introduction} %

\begin{figure*}[t]
    \centering
    \includegraphics[width=1\linewidth]{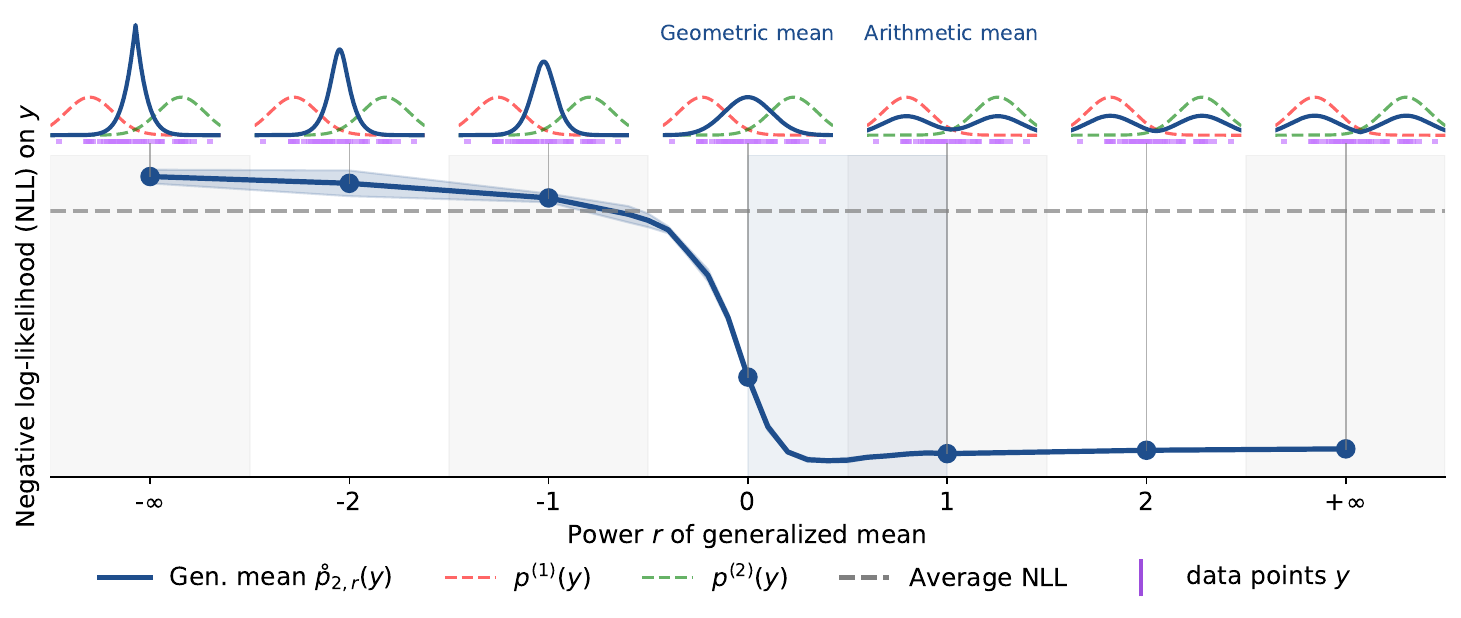}
    \caption{
    \textbf{Only intermediate powers $r$, particularly $r \in [0,1]$, yield consistent NLL improvements, whereas extreme values can degrade performance.}
    \textbf{Top:} aggregated densities $\compose{p}_2$ obtained from two Gaussian experts $p^{(1)}$ and $p^{(2)}$ for various $r$. Small $r$ concentrates mass in a single region, while large positive $r$ preserves the bimodal structure of the experts.
    \textbf{Bottom:} NLL evaluated on samples $y \sim \mathcal{N}(0,4^2)$ (purple ticks indicate test samples). In this setting, negative $r$ (min-type behavior) underperforms the average individual NLL (dashed line), whereas positive $r$ achieves the lowest values. Our theory (\Cref{theorem: log-likelihood inequality}) identifies $r \in [0,1]$ as the provably reliable interval. See Appendix~\ref{app: additional gaussian} for different behaviors. 
    }
    \label{fig:main figure}
\end{figure*}

Modern machine learning (ML) increasingly relies on collections of probabilistic models rather than a single monolithic predictor \citep{lakshminarayanan2017simple,liu2024deepseek}. This naturally raises the question of how to integrate several probability distributions into a single coherent distribution \citep{french1983group,genest1986combining}. 
Distribution integration arises in many settings in ML, including ensembling classifier outputs \citep[Section~5]{kuncheva2014combining}, Mixup training techniques \citep{el2025mixup}, Bayesian inference \citep{carvalho2023}, the aggregation of Gaussian process experts \citep{liu2018generalized,cohen2021healing}, mixture-of-experts architectures in large language models \citep{liu2024deepseek, cai2025survey}, and generative modeling \citep{liu2022compositional,razafindralambo2026two}. 

Two canonical approaches are particularly widespread. The first consists in forming a mixture of distributions \citep{jacobs1991adaptive,jordan1994hierarchical}, corresponding to an arithmetic average of densities. The second relies on multiplicative aggregation, where distributions are combined through a normalized product \citep{hinton2002training}. These methods offer contrasting properties  \citep{winkler1968consensus}: mixtures act as a logical `OR',  capturing heterogeneity by pooling the supports, whereas multiplicative models act as a logical `AND', sharpening the density where components overlap and thus favoring regions of consensus. This viewpoint has is recurrent in compositional generative modeling \citep{du2023reducereuserecyclecompositional,liu2022compositional}. More generally in ML, the choice between both has been a fundamental and practical question; for instance, \citet{buchanan2023effects} empirically compare logit and probability averaging in multi-class classification.

A key theoretical motivation for mixtures lies in their variance-reduction effect, allowing the collective evaluation to outperform the average individual ones (\citealp{hastie2009elements}, Chapter 15; \citealp{mattei2025ensembles}), a phenomenon related to the broader ``wisdom of crowds'' principle \citep{surowiecki2004wisdom}. When predictors are diverse, individual errors tend to compensate for one another, improving collective performance.

Beyond the classical arithmetic and geometric means, several works have sought to go beyond and generalize the notion of averaging probability distributions \citep{genest1986combining,amari2007integration} to reveal new aggregation properties and fundamental trade-offs underlying different pooling rules \citep{elkin2025opinion}. These efforts have led to continuous families of pooling operators that extend standard averages and include a broad spectrum of aggregation rules, ranging from harmonic and arithmetic means to more extreme operators such as max- or min-based pooling. In particular, \citet{amari2007integration} introduces an $\alpha$-family of operators parameterized by a free parameter $\alpha$, and establishes optimality properties with respect to $\alpha$-divergences for each member of this family.

We examine whether density aggregation through generalized means of order $r \in \overline{\R} \coloneqq \R \cup \{-\infty,+\infty\}$ \citep{hardy1952inequalities} provides improved aggregation rules for model ensembling. A natural criterion in ML to assess a aggregated distributions is the data log-likelihood, as it directly measures how well an aggregated model explains observed data. This objective is central to modern learning procedures, as it underlies cross-entropy loss minimization in supervised classification and KL-based objectives in broader settings such as generative modeling and variational inference (\citealp{bishop2006pattern}, Chapter 10; \citealp{theis2015note}; \citealp{goodfellow2016deep}, Chapter 5). This perspective therefore motivates both a theoretical and empirical investigation into whether variance reduction extends to generalized pooling operators and improve log-likelihood over individual models.

In this work, we show that our defined normalized generalized mean of order $r$ yields a reliable aggregation rule in terms of log-likelihood, with theoretical guarantees and empirical evidence that gains over individual distributions are consistently observed for $r \in [0,1]$. The idea is summarized in \Cref{fig:main figure}, which illustrates how the shape of normalized generalized mean density varies with $r$ along with the unreliable behavior of negative log-likelihood induced by negative $r$ values.

Our main contribution lies in analyzing, through the lens of log-likelihood, the generalized mean of order $r \in \overline{\R}$ which defines valid densities (\Cref{section: does every order r density}) with properties determined by $r$ (\Cref{section: background,section: aggregation of densities theory}). More precisely, this translates into the following results:
\begin{itemize}
\item We identify in \Cref{section: when provably reliable} a theoretical regime $r \in [0,1]$ ensuring consistent log-likelihood benefits for any set of densities, uniformly across all data points. This regime includes the arithmetic and geometric means as boundary cases, explaining their practical successes.
\item We characterize in \Cref{section: failure cases cex} failure points outside this interval based on agreement and disagreement of density values, highlighting distinct behaviors for $r<0$ and $r>1$. This explains the lack of reliability outside $[0,1]$ compared to the guarantees within the interval, as observed in \Cref{fig:main figure}. 
\item We empirically corroborate the theoretical findings through Deep Ensemble experiments on classification tasks on vision and sentiment analysis (\Cref{section: experiments}). We also observe that the optimal $r$ is non trivial but close to the $[0,1]$ range.
\end{itemize}

\section{Background and notations on generalized means}\label{section: background}
We consider a data space $\mathcal{X}$ equipped with a reference measure $\vd \vx$. We focus on probability distributions that admit densities with respect to $\vd \vx$. We denote by $\mathcal{P}$ the set of probability densities $p : \mathcal{X} \to \mathbb{R}_{\geq 0}$ satisfying $\int_{\mathcal{X}} p(\vx)\, \vd \vx = 1$.

Given $k \geq 1$ densities $p^{(1)}, \dots, p^{(k)} \in \mathcal{P}$, 
this work studies the construction of an aggregated density $\compose{p}$ obtained by combining these distributions through generalized averaging operators, followed in most of the cases by a normalization step to ensure that $\compose{p} \in \mathcal{P}$. 

We first introduce formal notation for the mixture and multiplicative aggregation, as they constitute the most popular reference aggregations for our results. We then turn to the generalized mean of order $r$ in $\overline{\R}$, which forms the central object of this paper and subsumes these operators.

\subsection{Canonical examples: linear and logarithmic pooling} \label{section: canonical examples}

We introduce below the two standard combinations for probability density functions $p^{(1)}$ and $p^{(2)}$ in $\mathcal{P}$, and recall a few of their basic properties.

The \emph{mixture} aggregation, also known as \emph{linear pooling / opinion pool} \citep{stone1961opinion,genest1986characterization,clemen1999combining, elkin2025opinion}, is defined for two probability as
\begin{equation}
\compose{p}_{\text{lin}}(\vx) \coloneqq \frac{1}{2}\Big(p^{(1)}(\vx) + p^{(2)}(\vx)\Big).
\end{equation}
This rule is highly democratic, as each density retains its influence in the aggregation, often resulting in a multimodal distribution. The mixture admits an appealing variational characterization \citep{amari2007integration}: it minimizes the average (forward) Kullback--Leibler divergence,
\begin{equation}\label{eq: mixture minimizes KL}
\compose{p}_{\text{lin}}
= \arg\min_{q \in \mathcal{P}} \frac{1}{2} \sum_{i=1}^2 \mathrm{KL}\!\left(p^{(i)} \,\|\, q\right).
\end{equation}

The \emph{multiplicative} aggregation is defined as the normalized product
\begin{equation}
\compose{p}_{\text{mult}}(\vx) \coloneqq \sqrt{p^{(1)}(\vx)\,p^{(2)}(\vx)}/Z,
\end{equation}
where $Z = \int_{\mathcal{X}} \sqrt{p^{(1)}(\vx)\,p^{(2)}(\vx)}\,\mathrm{d}\vx$ is the normalization constant. This operation is also commonly referred to as a \emph{Product-of-Experts} (PoE) \citep{hinton2002training,cohen2021healing, razafindralambo2026two}, 
but is also known in the literature as \emph{logarithmic pooling / opinion pool} \citep{carvalho2023,genest1986characterization,clemen1999combining, elkin2025opinion} or \emph{exponential mixture} \citep{amari2007integration}. The PoE is well-defined since the geometric mean of densities in $\mathcal{P}$ remains normalizable \citep{genest1986characterization,carvalho2023}.

The PoE enjoys several theoretical properties that sharply contrast with those of the mixture model. The PoE minimizes the average \emph{reverse} KL divergence to the individual models \citep{amari2007integration},
\begin{equation}\label{eq: PoE minimizes reverse KL}
\compose{p}_{\text{mult}}
= \arg\min_{q \in \mathcal{P}} \frac{1}{2} \sum_{i=1}^2 \mathrm{KL}\!\left(q \,\|\, p^{(i)}\right).
\end{equation} Moreover, it typically exhibits a more \emph{pessimistic} (and even dictatorial) behavior: regions assigned low density by any individual model are strongly penalized, and if $p^{(k)}(x)=0$ for some $k$, then $\compose{p}_{\text{mult}}(x)=0$ as well. As a result, PoE aggregations tend to produce distributions that are more concentrated, and often closer to unimodal compared to mixtures \citep{winkler1968consensus,genest1986combining}. \Cref{fig:main figure} provides an illustration (see $r=0$).

Other notable properties of the PoE include \emph{independence preservation} (\citealp{laddaga1977lehrer,cooke1991experts}, Chapter 11) and the \emph{externally} Bayesian property (\citealp{genest1986characterization,cooke1991experts}, Chapter 11). A more comprehensive overview of the properties for both linear and logarithmic pooling is also provided in \cite{elkin2025opinion} (Section 4).

Both aggregation schemes also have information-geometric interpretations: they correspond to the midpoints of two natural geodesics connecting $p^{(1)}$ and $p^{(2)}$ (see, e.g. \citealp{amari2016information}, Section 2.4).

\subsection{Generalized power mean} \label{section: generalized power mean}

We now discuss a generalized aggregation framework that extends the operators defined above. In this work, we adopt the generalized mean following the definition of \cite{hardy1952inequalities}, referred to as the \emph{ordinary mean}, and extend it to the setting of probability densities. Given $k$ positive values $a_1,\dots,a_k$, we define the \emph{generalized (or power) mean of order $r$} by
\begin{equation}
M_{r}(a_1,\dots,a_k) \coloneqq \left( \frac{1}{k} \sum_{i=1}^k a_i^r \right)^{1/r},
\quad r \in \mathbb{R} \setminus \{0\}.
\end{equation}
The continuous extension at $r = 0$ corresponds to the geometric mean, which naturally arises as the limit of $M_{r}$ when~$r$ tends to zero. We define
\begin{equation}
M_{0}(a_1,\dots,a_k) \coloneqq \left( \prod_{i=1}^k a_i \right)^{1/k}.
\end{equation}
The extreme cases recover the minimum and maximum values,
\begin{equation*}
M_{-\infty}(a_1,\dots,a_k) \coloneqq \lim_{r \to -\infty} M_{r}(a_1,\dots,a_k) = \min_i a_i,
\end{equation*}
\begin{equation*}\label{eq: Mr infty}
M_{+\infty}(a_1,\dots,a_k) \coloneqq \lim_{r \to +\infty} M_{r}(a_1,\dots,a_k) = \max_i a_i.
\end{equation*}

Classical choices recover well-known means: $r=1$ corresponds to the arithmetic mean (AM), $r=0$ to the geometric mean (GM), $r=-1$ to the harmonic mean. In fact the parameter $r$ controls the behavior of the generalized mean and can be interpreted as an optimism parameter. Large values of $r$ emphasize the largest inputs and lead to an optimistic aggregation, while small (negative) values emphasize the smallest inputs and result in a more pessimistic behavior. More precisely, larger values of $r$ make the mean a smooth approximation of the maximum, while smaller values make it a smooth approximation of the minimum. 

A key property of the generalized mean is its monotonicity with respect to the order $r$, which will play a central role in our analysis.

\begin{lemma}[\textbf{Power mean inequality, \citealp{hardy1952inequalities}; Chapter 2}]\label{lemma: p-mean inequality}
Let $a_1,\dots,a_k$ be positive values. If $r<s$, 
\begin{equation}
M_r(a_1,\dots,a_k) \leq M_s(a_1,\dots,a_k).
\end{equation}
where equality holds if and only if $a_1 = a_2 = \dots = a_k$. This result generalizes the inequality chain $\frac{a_1+a_2}{2} \geq \sqrt{a_1 a_2} \geq \frac{2}{\frac{1}{a_1} + \frac{1}{a_2}}$ where the left inequality is known as AM-GM.
\end{lemma}

\citet{hardy1952inequalities} provide an extensive list of elementary properties of $M_r$, including fundamental inequalities. Power means can be further generalized to $f$-means, where aggregation is defined through a transformation function $f$; see \citet{de2016mean} for a characterization of generalized means as $f$-means and their statistical properties.

\paragraph{Generalized means of probability densities.}We now extend the concept to the setting of probability densities. 
Given $k$ probability densities $p^{(1)},\dots,p^{(k)} \in \mathcal{P}$, we define 
\begin{equation}
M_{k,r}(\vx) \coloneqq M_{r}\big(p^{(1)}(\vx),\dots,p^{(k)}(\vx)\big).
\end{equation}
This yields a nonnegative function on $\mathcal{X}$, which does not necessarily integrate to one. For example, taking the geometric mean of two Gaussian densities does not in general produce a normalized density, and a normalization step is therefore required.

\begin{definition}[\textbf{Generalized power mean of densities}]\label{definition: generalized power mean of densities}
The mean of order $r$ associated with $p^{(1)},\dots,p^{(k)} \in \mathcal{P}$ is defined as
\begin{equation}
\compose{p}_{k,r}(\vx) \coloneqq \frac{1}{Z_{k,r}}\, M_{k,r}(\vx),
\end{equation}
where the normalization constant is given by
\begin{equation}
Z_{k,r} \coloneqq \int_{\mathcal{X}} M_{k,r}(\vx)\,\mathrm{d}\vx.
\end{equation}
\end{definition}

By construction, $\compose{p}_{k,r} \in \mathcal{P}$ whenever $Z_{k,r}$ is finite, which is true for all $r$ (see \Cref{proposition: generalized p-mean well defined}).

\subsection{Related work}

Closely related to our approach, \citet{cooke1991experts} study a family of density aggregation operators of the form $\compose{p}_{k,r} \propto M_{k,r}$, which coincides with our $r$-family of aggregations. They derive several immediate properties of this family, including properties analogous to those mentioned in \Cref{section: canonical examples} for the canonical cases, and \Cref{eq: case r equal 1}.

Differently, \citet{amari2007integration} introduce an \emph{$\alpha$-family of integrations}, where $\alpha \in \Z \cup \{ -\infty,+\infty \}$, based on the $\alpha$-representation of densities in information geometry  \citep{amari2000methods}, which also includes the classical means as special cases. They relate their mean to the $\alpha$-divergence
$D_{\alpha}(p \,\|\, q)$ for any $p,q \in \mathcal{P}$ \citep{chernoff1952measure,amari2000methods}, that include as special cases of $\alpha$ the KL, reverse KL, or the square of Hellinger. \citet[Theorem~2]{amari2007integration} show that the $\alpha$-integration is optimal under the $\alpha$-divergence criterion, in the sense that it minimizes the weighted average $\alpha$-divergence to the individual distributions. This result generalizes \Cref{eq: mixture minimizes KL,eq: PoE minimizes reverse KL}, and connects to earlier Bayesian formulations of generalized predictive densities \citep{corcuera1999generalized}.

Finally, a closely related line of work studies the general theory of distribution aggregation under the name of \emph{opinion pooling} \citep{elkin2025opinion}. In this framework, aggregation is defined abstractly as a parameter-free operator
$F : (P_1,\dots,P_n) \mapsto P^\star$,
mapping a profile of probability distributions to a collective one. This literature characterizes desirable axiomatic properties of such operators (e.g., unanimity preservation, independence, compatibility with Bayesian updating, ...).



\section{Likelihood guarantees for generalized mean aggregation}\label{section: aggregation of densities theory}

We now present the core theoretical contribution of this paper by analyzing the generalized power mean aggregation and by identifying when it guarantees a systematic improvement in log-likelihood. After establishing that the generalized mean defines a normalizable density in \Cref{section: does every order r density}, we prove in \Cref{section: when provably reliable} that such an improvement holds for orders $r \in [0,1]$, precisely between the geometric and arithmetic means, explaining their robustness. We finally show in \Cref{section: failure cases cex} via counterexamples that this guarantee fails outside this interval, with qualitatively different failure modes for $r<0$ and $r>1$.

\subsection{Does every order $r$ define a density?}\label{section: does every order r density}
An important point to examine is whether the normalization constant is finite for all values of $r$, and the following proposition shows that it is indeed the case. 
\begin{proposition}[\textbf{Generalized means of order $r$ are well defined}]\label{proposition: generalized p-mean well defined}
Assume $p^{(1)},\dots,p^{(k)}$ are $k$ positive probability density functions over $\mathcal{X}$. Then for all $r \in \overline{\R}$
\begin{equation}
\int_{\R^d} M_{k,r}(\vx) \vd \vx < \infty.
\end{equation}

\end{proposition}

\begin{proof}
We apply \Cref{lemma: p-mean inequality} followed by the definition of $M_{k,+\infty}$ from \Cref{eq: Mr infty}. For all $r \in \overline{\R}$,
\begin{align}
\int_{\R^d} M_{k,r}(\vx)\,\vd \vx
&\leq \int_{\R^d} M_{k,+\infty}(\vx)\,\vd \vx \\
&\leq \int_{\R^d} \sum_{i=1}^k p^{(i)}(\vx)\,\vd \vx = k < +\infty.
\end{align}
Additionally, it is worth noting that when $r \leq 1$ (for example for the geometric mean),
\begin{equation}\label{eq: case r equal 1}
\int_{\R^d} M_{k,r}(\vx) \vd \vx \leq \int_{\R^d} M_{k,1}(\vx) \vd\vx=1.
\end{equation}
\end{proof}
To the best of our knowledge, no previous work has established a bound on the normalizing constant for all real values of $r$. Proofs for the geometric pooling ($r=0$) appear in (\citealp{genest1986characterization}, p. 489, \citealp{carvalho2023}, Theorem 2.1). The case $r \leq 1$, also mentioned by \citet[Chapter~11]{cooke1991experts}, is of particular interest: the bound on the normalizing constant in \Cref{eq: case r equal 1} implies $-\log Z_{k,r} \geq 0$. The sign of $-\log Z_{k,r}$ is crucial for analyzing $\compose{p}_{k,r}$, since the log-likelihood can be written as $\log \compose{p}_{k,r} = \log M_{k,r} - \log Z_{k,r}$. It therefore directly determines whether normalization increases or decreases the overall likelihood, which is central to our arguments (see \Cref{theorem: log-likelihood inequality}).

\subsection{When is the generalized mean reliably beneficial?}
\label{section: when provably reliable}

\begin{figure*}[t]
    \centering
    \begin{subfigure}{0.33\textwidth}
        \centering
        \includegraphics[width=\linewidth]{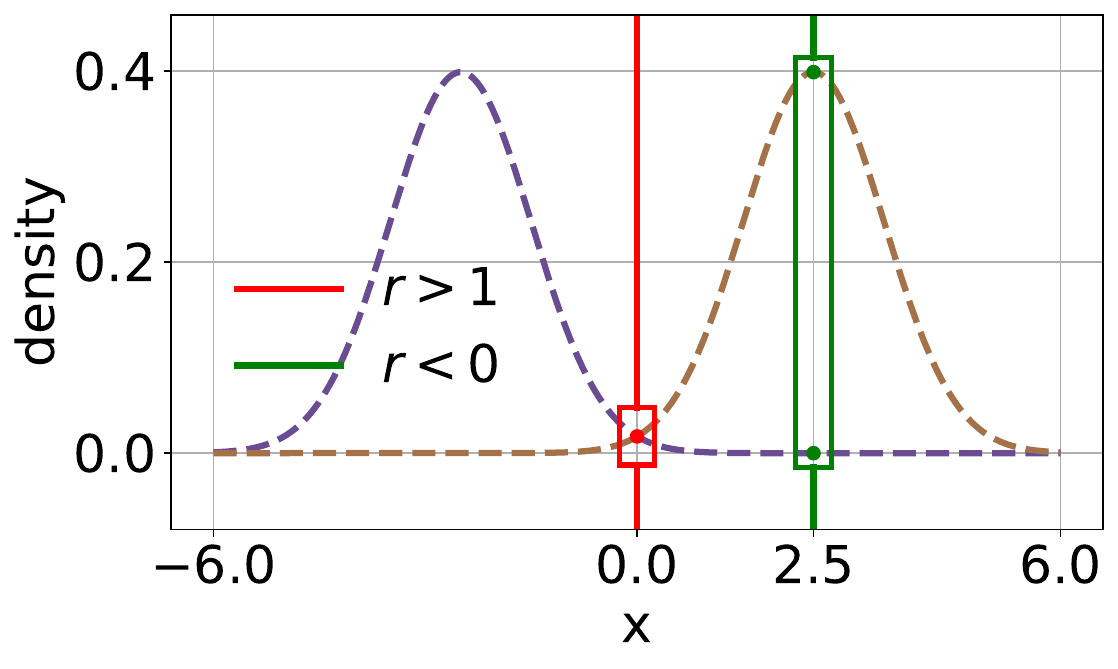}
        \caption{The two experts}
        \label{fig:counterexamples_gaussians}
    \end{subfigure}
    \hfill
    \begin{subfigure}{0.33\textwidth}
        \centering
        \includegraphics[width=\linewidth]{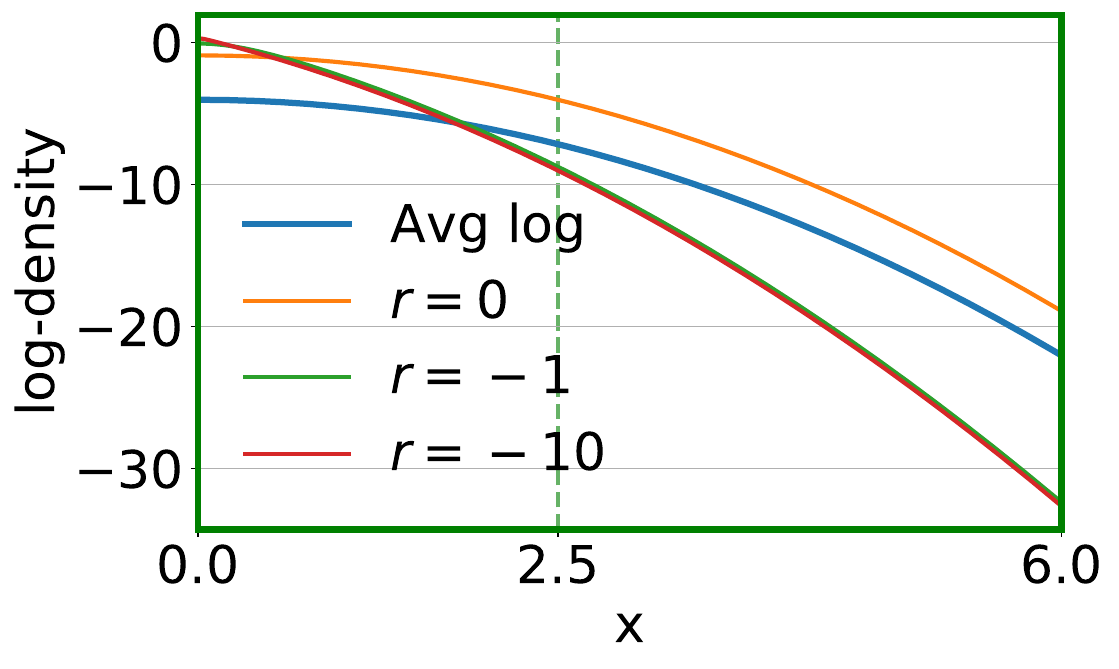}
        \caption{$r<0$}
        \label{fig:counterexamples_gaussians_m}
    \end{subfigure}
    \hfill
    \begin{subfigure}{0.33\textwidth}
        \centering
        \includegraphics[width=\linewidth]{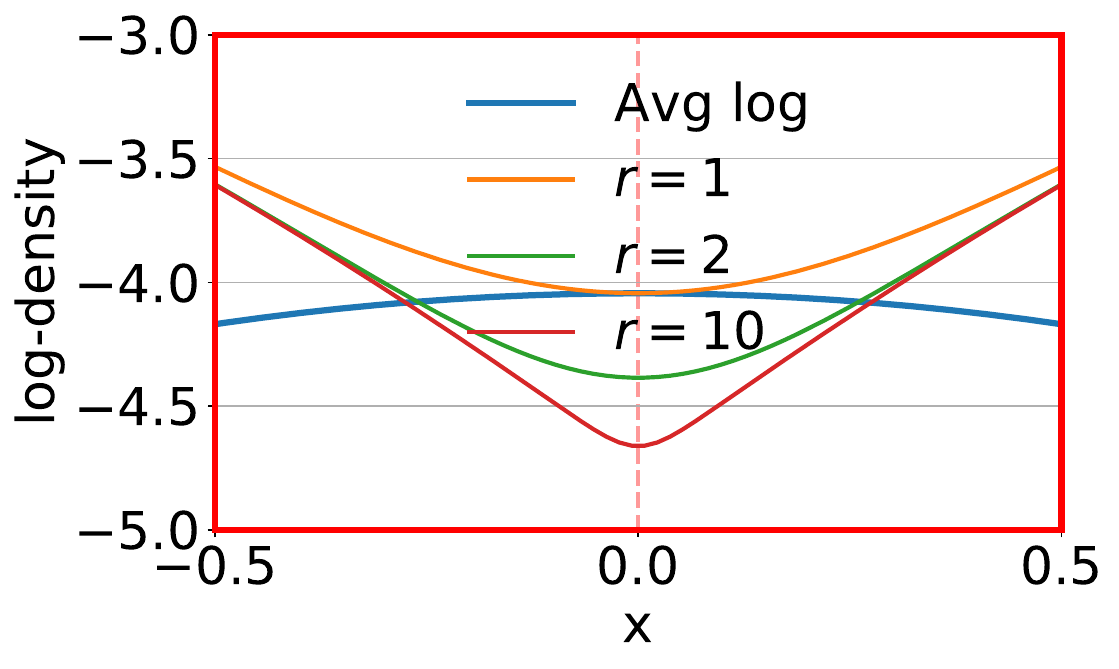}
        \caption{$r>1$}
        \label{fig:counterexamples_gaussians_0}
    \end{subfigure}
    
    \caption{
    \textbf{Visual illustration of \Cref{theorem: counterexamples}. Aggregation fails at points where densities strongly differ when $r<0$, and at consensus points when $r>1$.} We look at the location of the gap in \Cref{eq: bad inequality likelihood} across $r$ regimes. ``Avg log'' denotes the average of individual log-likelihoods (right-hand side of \Cref{eq: bad inequality likelihood}).
    Left (a): We consider two gaussians $\mathcal{N}(\pm2.5,1)$. Center (b): For $r<0$, the inequality of \Cref{theorem: log-likelihood inequality} is satisfied at $x=2.5$. We observe that the gap is not confined and amplifies when $x$ increases.
    Right (c): For $r \ge 1$, the AM ($r=1$) globally dominates the average log-density, while all $r>1$ exhibit a very localized (but non-punctual) negative effect around $x=0$.
    }
    \label{fig:counterexamples}
\end{figure*}

We examine the important question of when aggregation guarantees an improvement in comparison to the average individual log-likelihood, which would correspond to a ``wisdom of crowds'' effect, and show that this occurs exactly for $r \in [0,1]$. 

\begin{theorem}[\textbf{Wisdom of Crowds on Log-likelihood}]\label{theorem: log-likelihood inequality}
Let $\vx \in \mathcal{X}$ and $k \in \N^*$. Assume all density functions $p^{(1)},\dots,p^{(k)}$ are positive. If $0 \leq r\leq 1$,
\begin{equation}\label{eq: inequality log power mean}
\log \compose{p}_{k,r}(\vx) \geq \frac{1}{k} \sum_{i=1}^k \log p^{(i)}(\vx).
\end{equation}
\end{theorem}
\begin{proof}
Let $\vx \in \mathcal{X}$, and $0 < r \leq 1$. By Jensen’s inequality (concavity of $\log$) and since $r>0$,
\begin{align}
\log M_{k,r}(\vx)/Z_{k,r} & =  \frac{1}{r}\log \left( \frac{1}{k} \sum_{i=1}^k [p^{(i)}(\vx)]^r\right) - \log Z_{k,r} \\
& \geq \frac{1}{rk} \sum_{i=1}^k \log ([p^{(i)}(\vx)]^r) - \log Z_{k,r}. \label{eq: jensen inequality r positive}
\end{align}
Hence,
\begin{equation}\label{eq: final equation proof}
\log \compose{p}_{k,r}(\vx) \geq \frac{1}{k} \sum_{i=1}^k \log p^{(i)}(\vx) - \log Z_{k,r}. 
\end{equation}
Note that for $r=0$, \Cref{eq: final equation proof} still holds by using $\log \prod = \sum \log$.

Since $r \leq 1$, it is straightforward from \Cref{eq: case r equal 1} that
\begin{equation}
\log M_{k,r}(\vx)/Z_{k,r} \geq \frac{1}{k} \sum_{i=1}^k \log p^{(i)}(\vx)
\end{equation}
since $ \log Z_{k,r} \leq 0$.

\end{proof}

Log-likelihood quantifies how well a model $p$ explains data, and larger values indicating better fit. \Cref{theorem: log-likelihood inequality} therefore establishes aggregation reliability for $r \in [0,1]$, by guaranteeing that, for any data point $\vx$, the aggregated model achieves a log-likelihood at least as large as the average of individual ones. Thus $[0,1]$ can be seen as a ``safe'' interval.

The result is particularly revealing: the geometric mean~\mbox{($r=0$)} is the most pessimistic and least democratic aggregation that remains reliably beneficial, while the arithmetic mean ($r=1$) is the least pessimistic and most democratic one to achieve this. Together, they span the entire regime of safe aggregation, explaining why these two classical rules are privileged in the literature.

\subsection{Failure cases outside the reliability range}\label{section: failure cases cex}
The ``wisdom of crowds'' effect no longer holds when $r<0$ or $r>1$, as \Cref{eq: inequality log power mean} fails to hold for all $\vx$. This is formalized in the following theorem.

\begin{theorem}[\textbf{Failure of the Wisdom of Crowds}]\label{theorem: counterexamples}
Let $r \notin [0,1]$. Then there exist two positive densities
$p^{(1)}, p^{(2)} \in \mathcal{P}$ with $p^{(1)} \neq p^{(2)}$
and a point $\vx \in \mathcal{X}$ such that
\begin{equation}\label{eq: bad inequality likelihood}
\log \compose{p}_{2,r}(\vx)<\frac{1}{2} \sum_{i=1}^2 \log p^{(i)}(\vx).
\end{equation}

\end{theorem}

We establish \Cref{theorem: counterexamples} through \textbf{two different types of counter-examples} on $\vx$ depending on the \textbf{position of $r$} relative to the interval $[0,1]$. More precisely, when $r < 0$, \Cref{eq: inequality log power mean} breaks down at disagreement points, while for $r>1$ it fails at consensus points. As an illustrative example supporting this point and \Cref{theorem: counterexamples}, we consider two symmetric Gaussian densities defined as \begin{equation}
p^{(1)}(x)=\mathcal{N}(x;-m,1), 
\quad \text{and} \quad 
p^{(2)}(x)=\mathcal{N}(x;m,1).
\end{equation}

\paragraph{The aggregation breaks down at disagreement points.} When $r<0$, $x=m$ satisfies \Cref{eq: bad inequality likelihood}. This point corresponds to a pronounced imbalance, since for large $m$
\begin{equation}
p^{(2)}(m) \gg p^{(1)}(m) \approx 0,
\end{equation}
that is a location where the densities strongly disagree. The gap is in fact exponential, as \begin{equation}
\frac{p^{(2)}(m)}{p^{(1)}(m)}=e^{2m^2},
\end{equation} and this situation is illustrated in \Cref{fig:counterexamples_gaussians} in which $x=m$ corresponds to the mode of $p^{(2)}$ and the tail of $p^{(1)}$. Hence, it results in the average log-likelihood exceeding the log power mean on this point when $r<0$, as shown in \Cref{fig:counterexamples_gaussians_m}. This behavior arises because pessimistic (min-like) power means heavily penalize regions where at least one density assigns nearly zero probability mass. A detailed proof is provided in Appendix~\ref{app: symmetric gaussians counterexample}, where we show that this penalization effect for large $m$ reverses Jensen's inequality from \Cref{eq: jensen inequality r positive}.

\paragraph{The aggregation breaks down at consensus points.} When $r>1$, we show in contrast that $x=0$ satisfies \Cref{eq: bad inequality likelihood}. 
This corresponds to a positive consensus point, namely a location where both densities agree and assign nonzero probability mass. Indeed, at this point,
\begin{equation}
p^{(1)}(0)=p^{(2)}(0)>0,
\end{equation}
while the two densities differ elsewhere. Such a situation is illustrated in \Cref{fig:counterexamples_gaussians}. 
As shown in \Cref{fig:counterexamples_gaussians_0}, the average log-likelihood exceeds $\compose{p}_{2,r}$ at this point when $r>1$. Unlike the previous case, this agreement is localized, resulting in a highly local breakdown of the inequality. 
The failure originates from the normalization: although max-like aggregation favors large values, normalization redistributes probability mass toward regions where one density dominates, thereby weakening the contribution of the consensus point. A detailed proof is provided in Appendix~\ref{app: symmetric gaussians counterexample}, where we show explicitly that $Z_{2,r}$ enforces \Cref{eq: bad inequality likelihood} even though Jensen's inequality holds with equality.

\begin{figure*}[t]
    \centering
    \begin{subfigure}{0.32\textwidth}
        \centering
        \includegraphics[width=\linewidth]{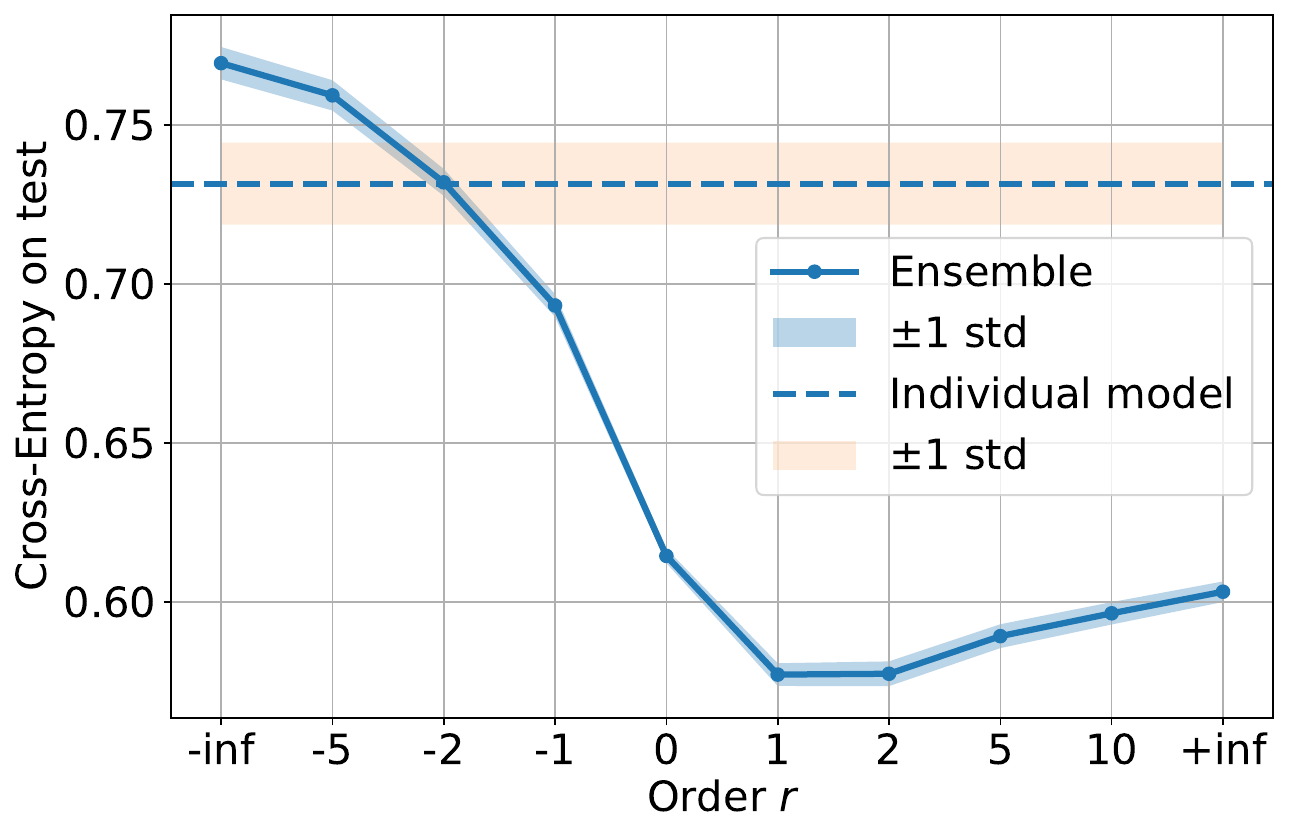}
        \caption{CIFAR-100}
        \label{fig:cifar_triptych_a}
    \end{subfigure}
    \hfill
    \begin{subfigure}{0.32\textwidth}
        \centering
        \includegraphics[width=\linewidth]{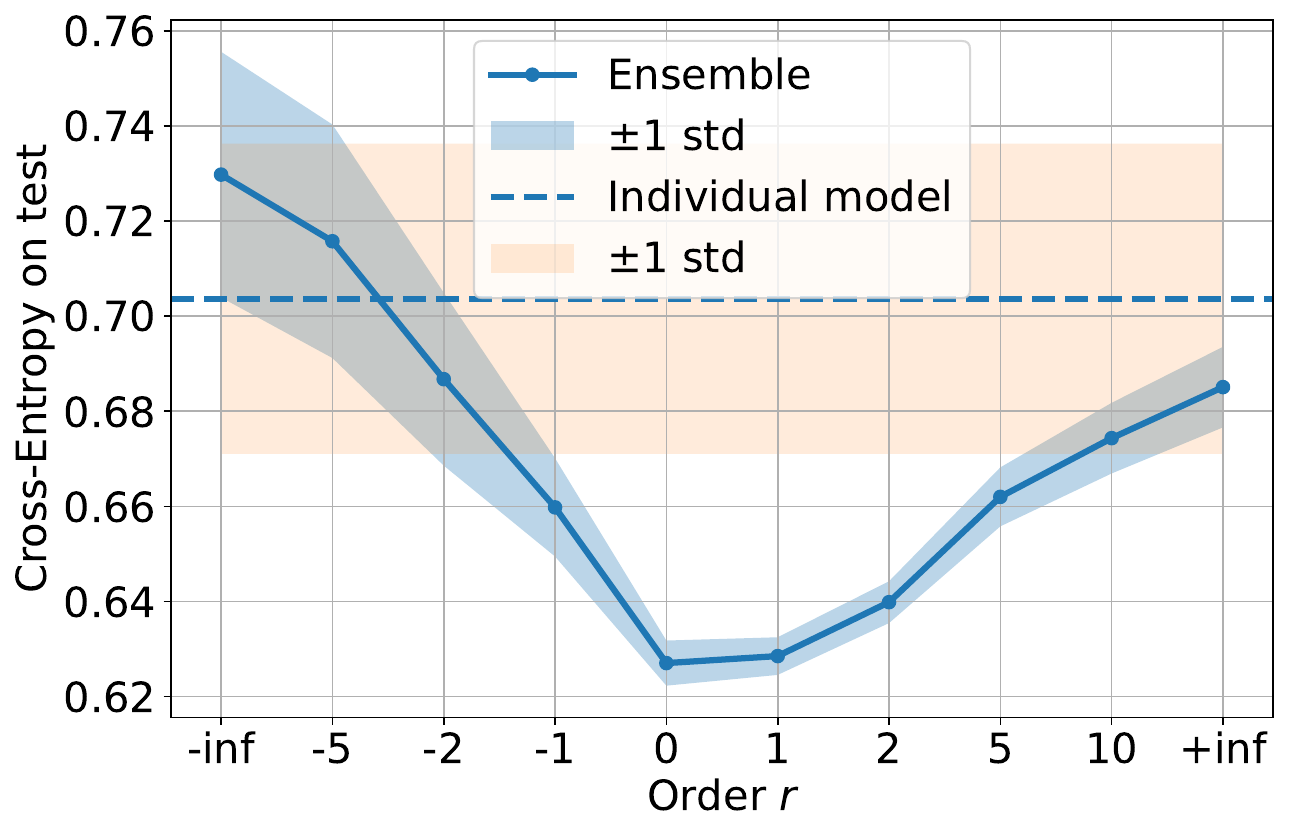}
        \caption{MedMNIST}
        \label{fig:medmnist_a}
    \end{subfigure}
    \hfill
    \begin{subfigure}{0.32\textwidth}
        \centering
        \includegraphics[width=\linewidth]{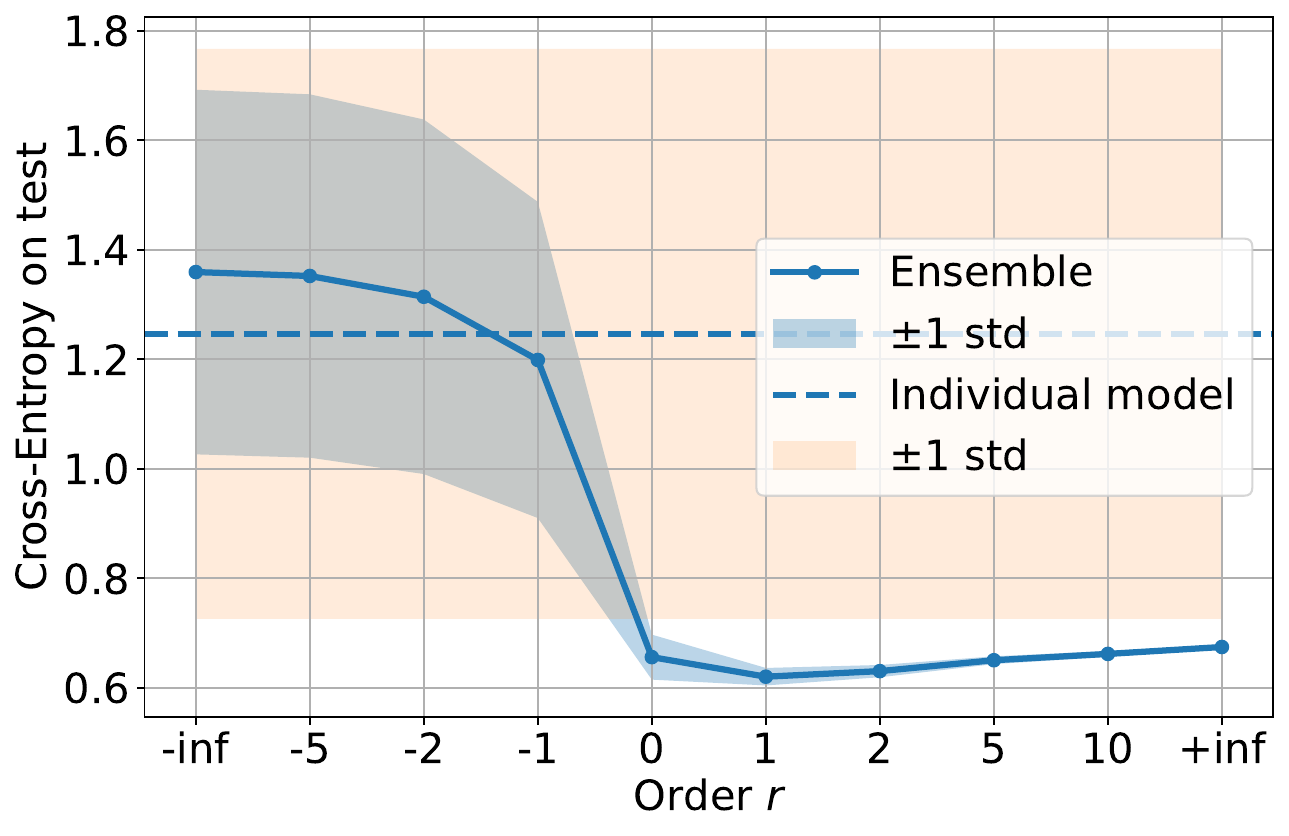}
        \caption{IMDb}
        \label{fig:imdb_a}
    \end{subfigure}
    \caption{\textbf{Illustration of the cross-entropy behavior of the normalized power mean likelihood $\compose{p}_{k,r}$ (via cross-entropy) in three classification settings.} All plots exhibit a U-shaped performance curve: extreme aggregation amplifies model disagreement, whereas optimal values lie in $[0,1]$. Consistent with \Cref{theorem: log-likelihood inequality}, the regime $r \in [0,1]$ remains reliably below the individual model uncertainty band. By contrast, negative orders ($r<0$) perform poorly, likely due to disagreements on dominant classes, as discussed in \Cref{section: failure cases cex}.
    }
    \label{fig:global_r}
\end{figure*}
\begin{figure*}[t]
    \centering
    \begin{subfigure}{0.32\textwidth}
        \centering
        \includegraphics[width=\linewidth]{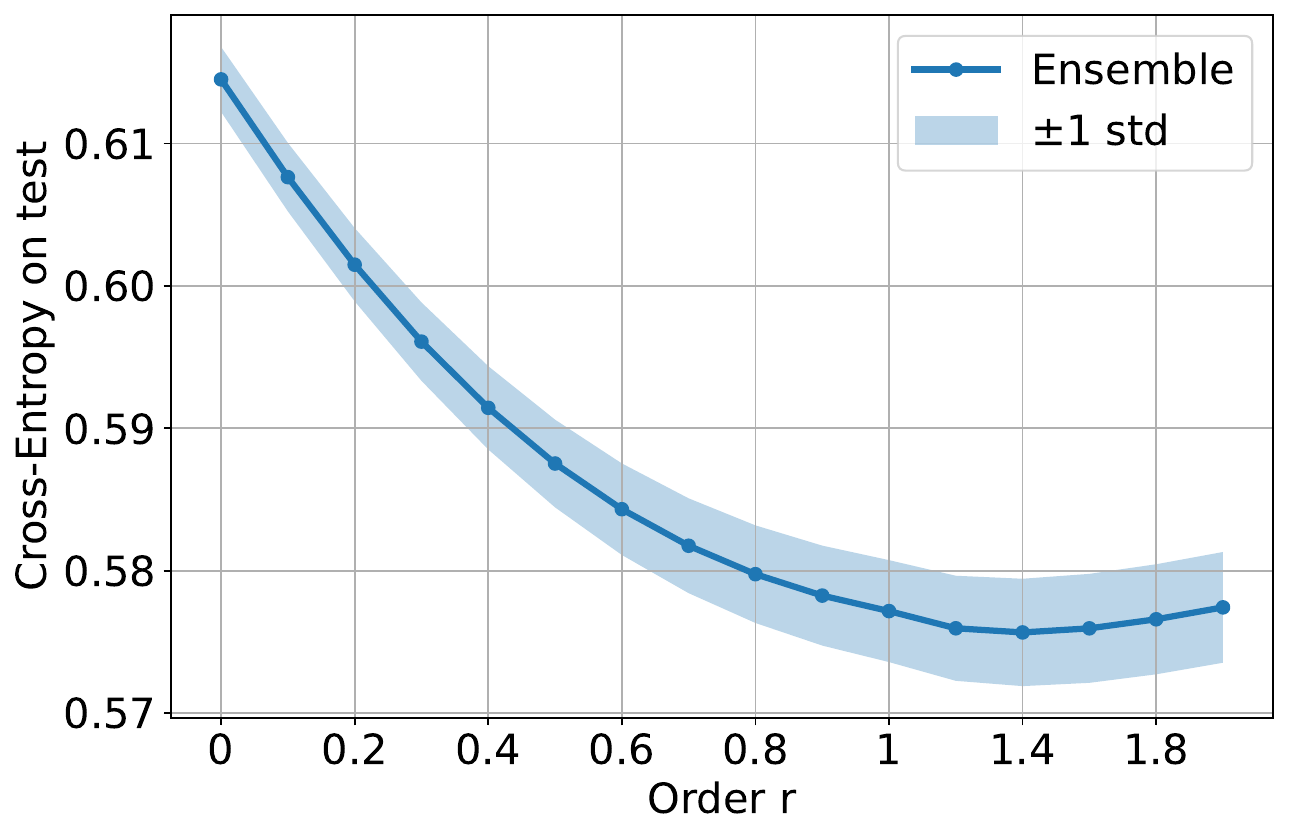}
        \caption{CIFAR-100}
        \label{fig:cifar_triptych_b}
    \end{subfigure}
    \hfill
    \begin{subfigure}{0.32\textwidth}
        \centering
        \includegraphics[width=\linewidth]{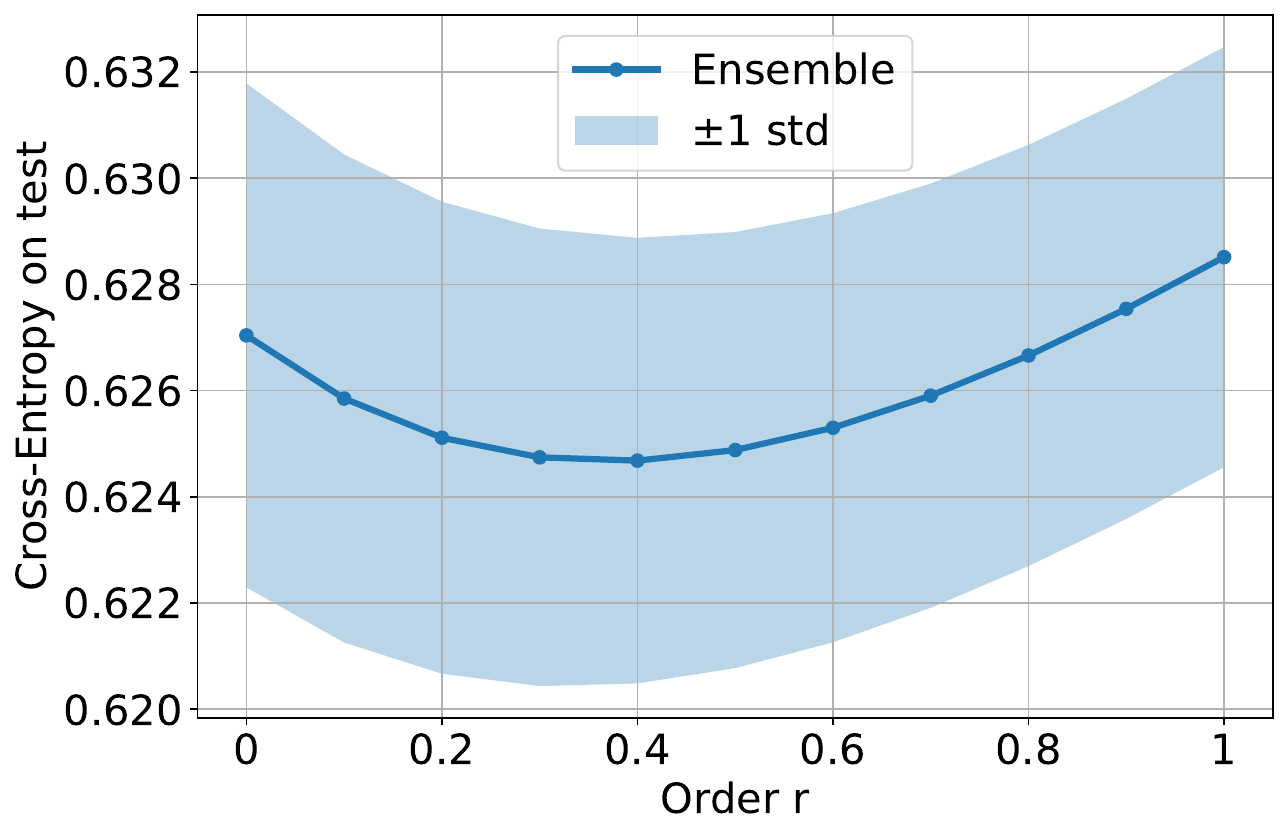}
        \caption{MedMNIST}
        \label{fig:medmnist_b}
    \end{subfigure}
    \hfill
    \begin{subfigure}{0.32\textwidth}
        \centering
        \includegraphics[width=\linewidth]{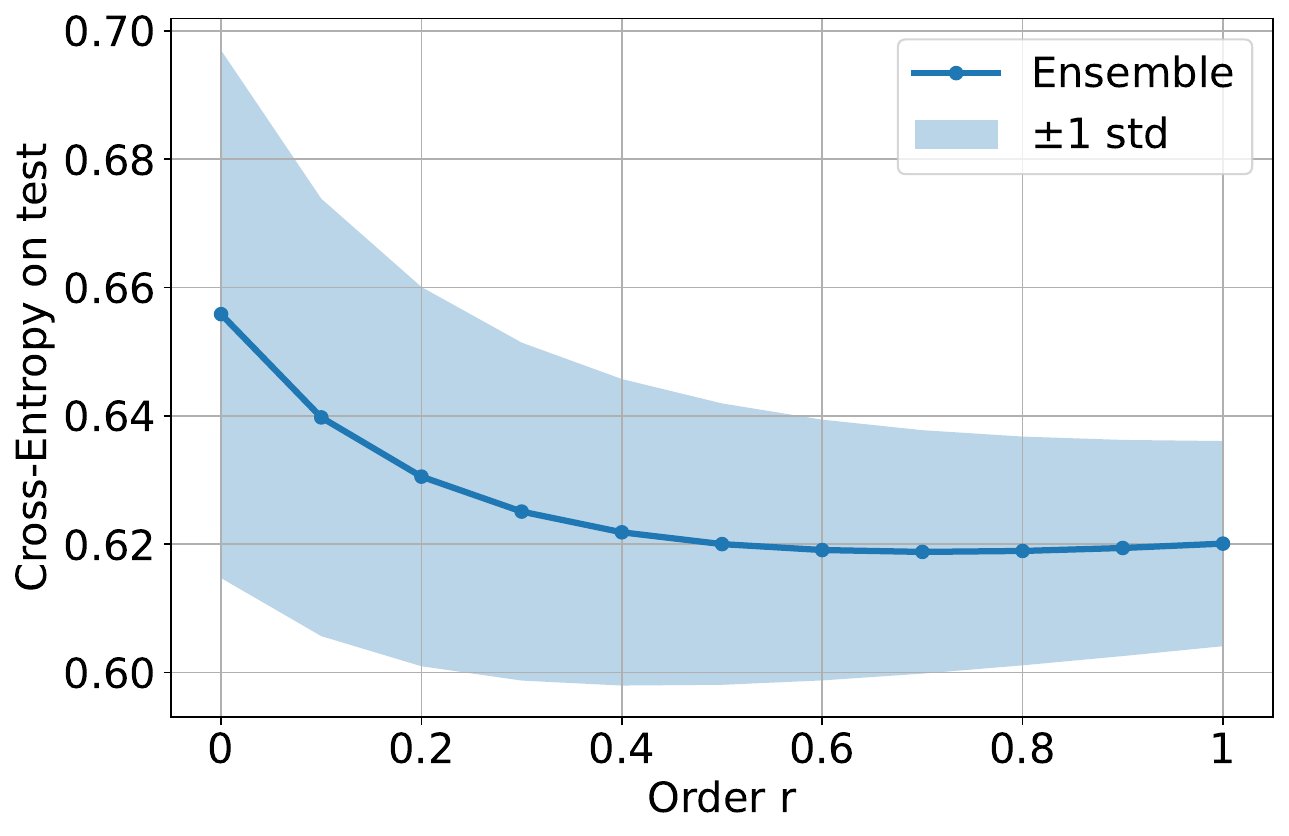}
        \caption{Imdb}
        \label{fig:imdb_b}
    \end{subfigure}
    \caption{\textbf{Illustration of the cross-entropy behavior of the normalized power mean likelihood $\compose{p}_{k,r}$, focusing on values of $r$ around $[0,1]$ (zoomed view of \Cref{fig:global_r}).} On MedMNIST (b) and IMDb (c), the optimal likelihood lies within the reliability interval $[0,1]$ (\Cref{theorem: log-likelihood inequality}), whereas on CIFAR-100 (a) it lies slightly beyond it ($\approx 1.4$). This shows that while the interval $[0,1]$ provides a stable and reliable regime, mild optimism outside it can still be beneficial in practice.
    }
    \label{fig:local_r}
\end{figure*}


\textbf{Non-continuous counter-example.} This phenomenon is even clearer in the extreme regimes and can be illustrated in the discrete (classification) setting.
We can show that for $r=-\infty$ (minimum), two classifiers over $y\in\{0,1\}$ with true label $y=0$, 
given by $p^{(1)}=(0.9,\,0.1)$ and $p^{(2)}=(0.01,\,0.99)$, form a counter-example.
Similarly, we can show that for $r=+\infty$ (maximum), two distinct classifiers over $y\in\{0,1,2\}$ with true label $y=0$, 
given by $p^{(1)}=(0.9,\,0.03,\,0.07)$ and $p^{(2)}=(0.9,\,0.07,\,0.03)$, also violate the inequality. See Appendix~\ref{appendix:extreme_counterexamples} for derivation.

\section{Experiments}\label{section: experiments}

In this section, we empirically test whether generalized mean aggregation within the theoretical safe regime $r \in [0,1]$ from \Cref{theorem: log-likelihood inequality} yields actual improvements in negative log-likelihood across all models and tasks considered, and analyze how behavior degrades outside this range.
\paragraph{How to build models to ensemble.}To study aggregation behavior in practice, we train multiple neural networks with identical architectures independently on the same dataset and loss, using different random initializations and data shuffles. This setup corresponds to the \textit{deep ensemble} (DE) framework introduced by \citet{lakshminarayanan2017simple} and \citet{fort2019deep}. In this case, each predictive distribution $p^{(i)}$ is represented by the softmax output of an independently trained classifier, and aggregation is performed over multiple such models. For each experiment, we build ensembles of size 10 and repeat the process across 5 independent ensemble draws to reduce variance. Aggregation is applied at inference time using the generalized mean of order $r$ (\Cref{definition: generalized power mean of densities}).

\paragraph{Datasets and architectures.} We evaluate aggregation across vision, medical imaging, and NLP to test the generality of the phenomenon. For multi-class classification, we use CIFAR-100 \citep{krizhevsky2009learning}, a benchmark of 60,000 low-resolution color images spanning 100 classes with only 600 images per class. The large label space and limited per-class data increase the classification difficulty and the likelihood of inter-model disagreement. In the medical domain, we consider DermaMNIST \citep{yang2023medmnist}, which contains dermatoscopic skin lesion images across  diagnostic categories. The dataset is strongly imbalanced, with benign melanocytic nevi representing about two thirds of samples while several classes account for only 1–5\%. We train standard convolutional residual networks \citep{he2016deep,zagoruyko2016wide} for these vision tasks. Finally, to assess aggregation in NLP, we train transformers \citep{vaswani2017attention} on IMDb \citep{maas2011learning}, a balanced binary sentiment classification dataset of 50,000 movie reviews. See Appendix~\ref{app:experimental details} for further details.

\paragraph{Evaluation.} We evaluate performance using test cross-entropy, which corresponds to the negative log-likelihood of the true labels under the predictive distribution. We measure how it varies as a function of $r$, with particular attention to the predicted safe regime $r \in [0,1]$.

\subsection{Intermediate orders are more reliable than extremes}

By analyzing aggregation across the full range $r \in \overline{\R}$, we show in \Cref{fig:global_r} that performance follows a characteristic  U-shaped curve, with reliable improvements confined to the regime $r \in [0,1]$.

Across all three datasets extreme orders degrade performance, whereas intermediate values, in particular $[0,1]$ (cf.~\Cref{theorem: log-likelihood inequality}), consistently outperform individual models. This pattern holds even under strong inter-model uncertainty (\Cref{fig:imdb_a}), where all $r \in [0,1]$ remain below the uncertainty band, supporting the reliability of this interval.

At the extremes, $r = -\infty$ yields worse cross-entropy than individual models, likely due to amplified disagreement (\Cref{section: failure cases cex}), especially on CIFAR-100 where the large label space increases conflicting predictions. Conversely, on DermaMNIST, $r = +\infty$ reaches the individual uncertainty region, suggesting that max-like aggregation can amplify shared confident errors likely induced by class imbalance.

Finally, ensemble variance is reduced relative to individual models for all values of $r$. Moreover it remains larger for $r<0$, and decreases for $r \geq 0$, indicating that larger values (in particular intermediate values) of $r$ yield stronger variance reduction while achieving the best overall log-likelihoods, supporting the ``wisdom of crowds'' effect.
\begin{figure}
    \centering
    \includegraphics[width=.9\linewidth]{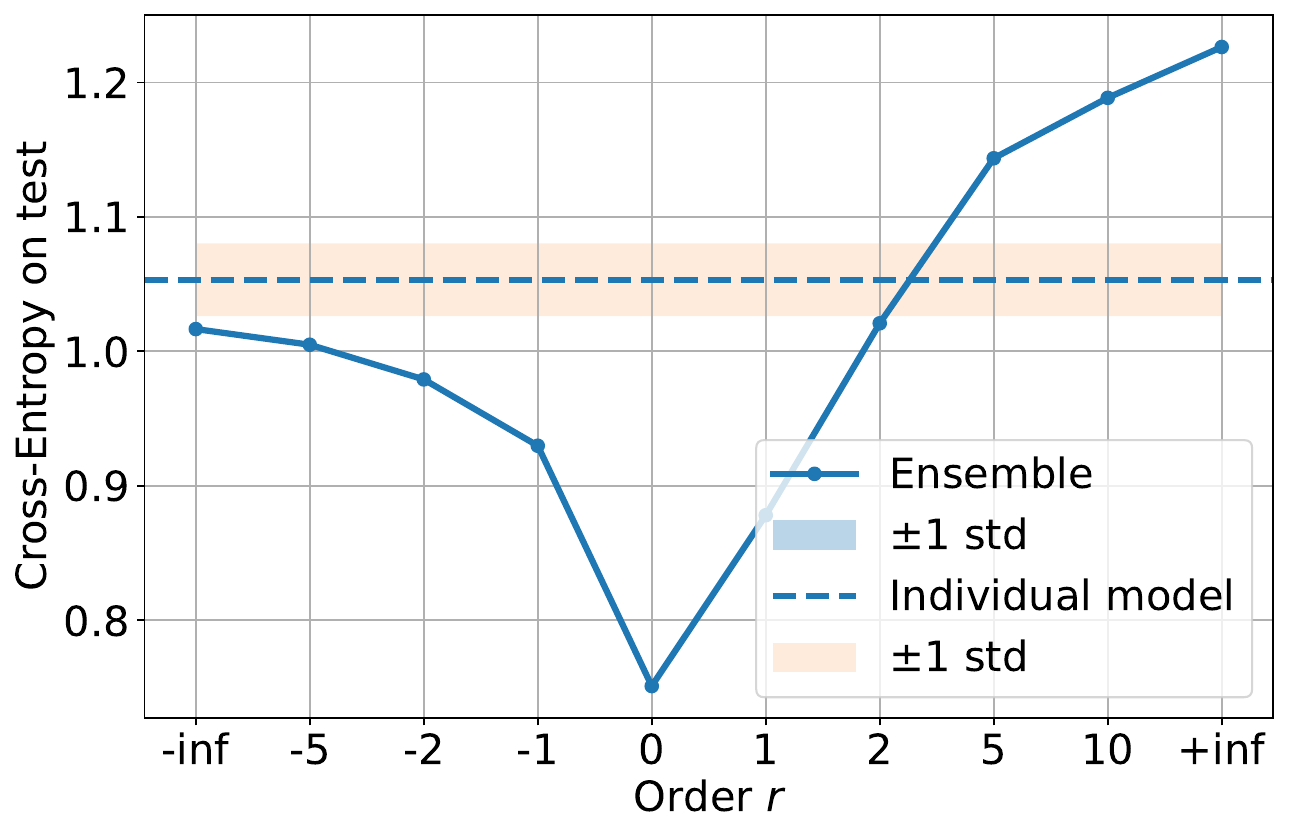}
    \caption{\textbf{Illustration of the cross-entropy behavior of the normalized power mean likelihood $\compose{p}_{k,r}$ on CIFAR-100 in a controlled near-consensus regime.} The trend contrasts with \Cref{fig:global_r}: extreme optimistic aggregation becomes harmful and the maximum operator performs worst.}
    \label{fig:cifar_triptych_c}
\end{figure}
\subsection{What is the optimal order $r$?}

We reveal through a finer sweep on each plot (\Cref{fig:local_r}) that the optimal $r$ is typically intermediate rather than located at extreme values, but that despite $I=[0,1]$ being the theoretically reliable regime, the empirically optimal $r$ value can sometimes lie outside of it.

In \Cref{fig:medmnist_b,fig:imdb_b}, the empirical optimum lies in $[0.3,0.5] \subset I$ and $[0.6,1] \subset I$, respectively. By contrast, \Cref{fig:cifar_triptych_b} suggests that mild optimism ($r \in [1.2,1.6] \nsubseteq I$) can still improve model likelihood in practice. 

This shows that the interval $[0,1]$ does not guarantee optimality, as better empirical performance may occur for values of $r$ exceeding $1$. Related work by \citet{safont2019multiclass} learns the optimal aggregation parameter directly from data for classification, within an alternative family of aggregation rules derived from $\alpha$-divergences.

We note that no universal ordering exists between the canonical values $r=0$ and $r=1$, as neither consistently dominates across datasets as observed in \Cref{fig:local_r}. Interestingly, the arithmetic mean ($r=1$) appears less favorable on the imbalanced dataset, in line with observations reported by \citet{buchanan2023effects} in predictive settings.

\subsection{Near-consensus makes optimism harmful}

We show in \Cref{fig:cifar_triptych_c} that aggregation behavior can reverse in comparison to \Cref{fig:cifar_triptych_a,fig:medmnist_a,fig:imdb_a} when models assign nearly equal probability to one class while differing elsewhere: extreme optimistic rules become harmful, and the maximum-type aggregation performs worst. This supports the counter-example in \Cref{section: failure cases cex}.

Such structured near-consensus is difficult to obtain with independently trained predictors, as agreement typically fluctuates across classes and inputs. To isolate this effect and get \Cref{fig:cifar_triptych_c}, we construct a controlled CIFAR-100 experiment by duplicating a trained predictor, and then adding small Gaussian perturbations to its logits except the true-class logit (while ensuring that the probabilities sum to one).


\section{Conclusion}

We studied the impact of aggregating probability distributions through the normalized generalized mean of order $r$ on densities, which defines a continuous family encompassing classical mixture and geometric aggregation rules as special cases. We theoretically show that $r \in [0,1]$ is the regime ensuring systematic improvements of log-likelihood over individual models, while outside orders may fail with distinct mechanisms. Simple Gaussian settings and experimental validation on classification tasks corroborate these results. Notably, we confirm the distinguished roles of the geometric and arithmetic means, whose widespread use aligns with their delimitation of the theoretical safe regime $[0,1]$. \\ Our work evaluates density operations through log-likelihood, while prior studies focus on accuracy. For instance, \citet{buchanan2023effects} compare averaging normalized logits and softmax outputs (logarithmic vs.\ linear pooling) and observe dataset-dependent behavior, with logarithmic pooling ($r=0$) favored on long-tailed datasets. This confirms the role of data in aggregation performance.\\ While our results identify the interval $r \in [0,1]$ as a provably reliable regime, experiments suggest that the empirically optimal value may depend on properties of the data or model. Consequently, a natural direction for future work is to better characterize the optimal aggregation order $r$.                                      

\begin{acknowledgements}
This work was supported by the French government through the France 2030 programme managed by the ANR.
PAM was also supported by the French government, managed by the ANR, with the reference number “ANR-23-IACL-0001”, as well as RS with the reference number “ANR-15-IDEX-01” and DG through
project NIM-ML “ANR-21-CE23-0005-01”.
\end{acknowledgements}

\bibliography{uai2026-template}

\newpage

\onecolumn

\title{Beyond Mixtures and Products for Ensemble Aggregation: \\ A Likelihood Perspective on Generalized Mean
\\ (Supplementary Material)}
\maketitle

\appendix

This appendix complements the main paper with methodological details, supplementary illustrations, and additional theoretical derivations:
\begin{itemize}
    \item Appendix~\ref{app:experimental details} presents detailed experimental settings corresponding to the experiments of \Cref{section: experiments}.
    
    \item Appendix~\ref{app: additional gaussian} explores a broader range of behaviors of generalized means in the Gaussian setting, extending the illustrative example of \Cref{fig:main figure}.
    
    \item Appendix~\ref{app: proof counterexample} provides theoretical derivations of the ``wisdom of crowds'' counterexamples introduced in \Cref{section: failure cases cex}.
    
    \item Appendix~\ref{app: analytical expressions} derives analytical expressions for the normalization constant in the $\R^d$ Gaussian case with $k \geq 1$ densities.
\end{itemize}

\section{Experimental details}\label{app:experimental details}

We describe the experiments of \Cref{section: experiments} in greater detail, covering both the experimental setup and the numerically stable computation procedures. We will provide the code upon acceptance.

\subsection{Datasets and models}

We adopt three datasets, each paired with an architecture suited for the associated classification task.

We use CIFAR-100 \citep{krizhevsky2009learning}, a benchmark dataset of 60,000 color images of size $32\times32$ spanning 100 object classes, with only 600 images per class. The classes cover a wide variety of natural objects, including animals, vehicles, household items, and everyday scenes. Our training set contains 50k images. We train WideResNet-28-10 models \citep{zagoruyko2016wide}, where 28 denotes the network depth and 10 the widening factor controlling the number of channels. It is a popular architecture for small-resolution image classification tasks.

We also evaluate DermaMNIST \citep{yang2023medmnist}, derived from the HAM10000 collection \citep{tschandl2018ham10000}, a medical imaging dataset of dermatoscopic skin lesion images categorized into seven diagnostic classes. The dataset is relatively small and strongly class-imbalanced: benign melanocytic nevi represents about two thirds of samples while several classes account for only 1–5\%. We adopt the official dataset split ($\approx$ 7k train / 1k val / 2k test). We use the standard $28\times28$ resolution version and adopt a ResNet-18 architecture, following the experimental setup recommended by the MedMNIST benchmark \citep{yang2023medmnist}.

We finally evaluate on the IMDb movie reviews dataset \citep{maas2011learning}, a large-scale natural language corpus of 50,000 user-written reviews from the Internet Movie Database, where the task consists in classifying reviews as positive or negative. We train lightweight Transformer-based text classifiers composed of token embeddings, sinusoidal positional encoding, and a small Transformer encoder followed by mean pooling and a linear classification head. We use a 35k/7.5k/7.5k train/validation/test split.

\subsection{Training the ensemble}

We train five independent Deep Ensembles \citep{lakshminarayanan2017simple} per experiment, each composed of ten models trained with cross-entropy loss. Models within each ensemble are trained independently with different random initializations and data shuffling. The best checkpoints are selected based on validation performance. Reported means and standard deviations are computed across the five ensembles to assess statistical stability.

\subsection{Evaluating the generalized mean for classification}

In this section, we explain how generalized mean aggregation is implemented on softmax outputs for likelihood evaluation in classification.

In this setting, each model produces a predictive distribution over $L$ classes. 
Given $k$ predictors, we denote by
\begin{equation}
p^{(i)}(y=\ell \mid x), \qquad i=1,\dots,k,
\end{equation}
the probability assigned to class $\ell$ for an input $x$. In this case, averaging likelihoods amounts to aggregating the predicted class probabilities themselves. We therefore apply the generalized (power) mean of order $r$ across model predictions. 
For each class $y$, we define
\begin{equation}
M_{k,r}(y \mid x)
=
\begin{cases}
\left(
\frac{1}{k}
\sum_{i=1}^{k}
\big(p^{(i)}(y \mid x)\big)^r
\right)^{\frac{1}{r}}, & r \neq 0, \\[6pt]
\exp\!\left(
\frac{1}{k}
\sum_{i=1}^{k}
\log p^{(i)}(y \mid x)
\right), & r=0 .
\end{cases}
\end{equation} Because this operation is applied independently to each class, the resulting scores do not necessarily sum to one. We therefore follow the necessary normalization step across classes to obtain a valid predictive distribution,
\begin{equation}
\compose{p}_{k,r}(y \mid x)
=
\frac{M_{k,r}(y \mid x)}
{\sum_{y=1}^{L} M_{k,r}(y \mid x)}.
\end{equation}

\paragraph{Stable log-domain implementation.}
Direct computation in probability space may lead to numerical instability due to extremely small softmax probabilities and the normalization step involving sums of exponentials. We therefore perform aggregation in the log-domain using numerically stable $\mathrm{logsumexp}$ operations (e.g., \texttt{torch.logsumexp}\footnote{Implemented using the PyTorch library \citep{paszke2019pytorch}.}).

For finite $r \neq 0$, the generalized mean of predictions writes
\begin{equation}
\log M_{k,r}(y \mid x)
=
\frac{1}{r}
\left(
\log
\sum_{i=1}^{k}
\exp\!\big(r \log p^{(i)}(y \mid x)\big)
-
\log k
\right),
\end{equation}
which can be implemented efficiently using a $\mathrm{logsumexp}$ operation over models.

For $r=0$, we write the geometric mean as
\begin{equation}
\log M_{k,0}(y \mid x)
=
\frac{1}{k}
\sum_{i=1}^{k}
\log p^{(i)}(y \mid x).
\end{equation}

The predictive distribution is finally obtained by exponentiating and normalizing across classes,
\begin{equation}
\compose{p}_{k,r}(y \mid x)
=
\exp\!\big(
\log M_{k,r}(y \mid x) - \log Z(x)
\big),
\end{equation}
where the normalization constant is
\begin{equation}
\log Z(x)
=
\log
\sum_{y}
\exp\!\big(
\log M_{k,r}(y \mid x)
\big).
\end{equation}

\section{Visual experiments on the Gaussian setting
}\label{app: additional gaussian}

In this section, we analyze the continuous setting (that would correspond to regression) to illustrate how both expert configuration and data dispersion influence the effect of the power parameter $r$ of $\compose{p}_{k,r}$ on the negative log-likelihood.

We consider symmetric Gaussian experts $p^{(1)}$ and $p^{(2)}$ of means $\mu_1$ and $\mu_2 = -\mu_1$ respectively with fixed standard deviation $1.8$, and evaluate aggregation performance on samples drawn from a centered Gaussian distribution with standard deviation $\sigma$ (with $\sigma=4$ corresponding to \Cref{fig:main figure}). For each setting, the NLL is estimated over several tens of thousands of samples. By varying $\sigma$ and the separation between experts, we clarify the mechanisms governing the different aggregation regimes.

\subsection{How data concentration shapes aggregation performance}
\begin{figure}[t]
    \centering
    
    \includegraphics[width=0.75\linewidth]{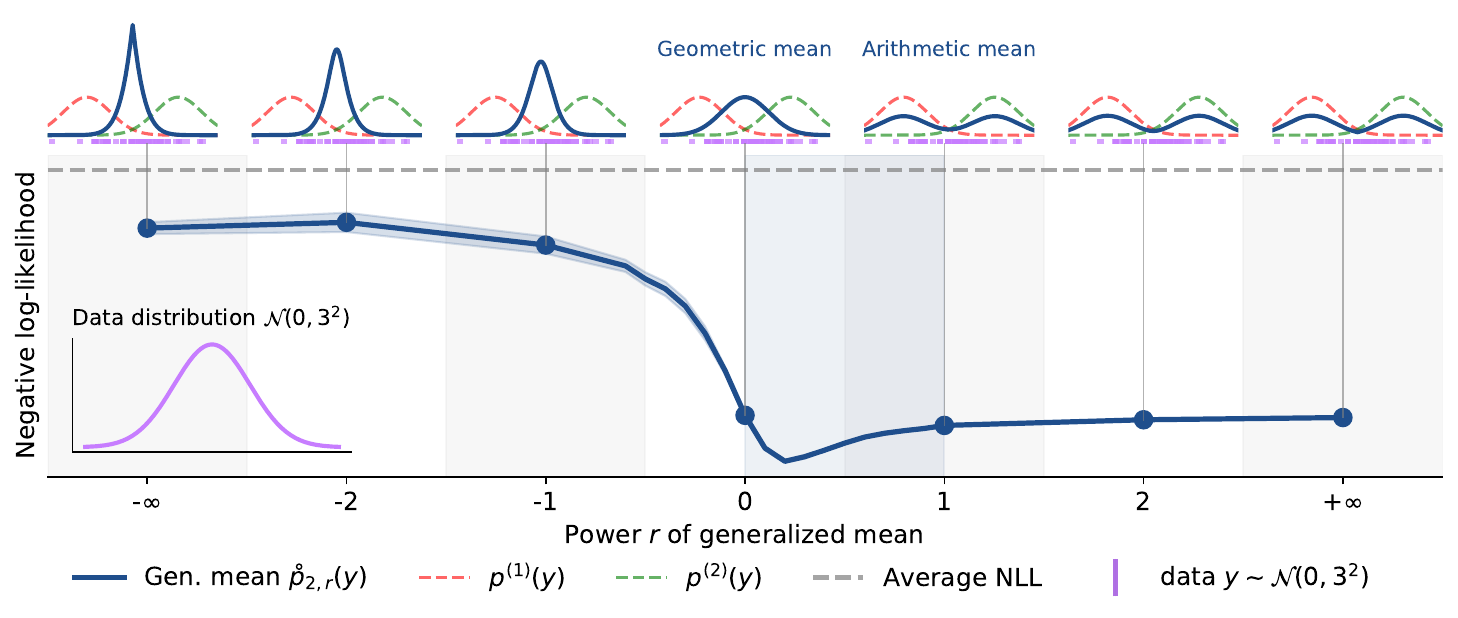}
    
    \vspace{0.5em}
    \includegraphics[width=0.75\linewidth]{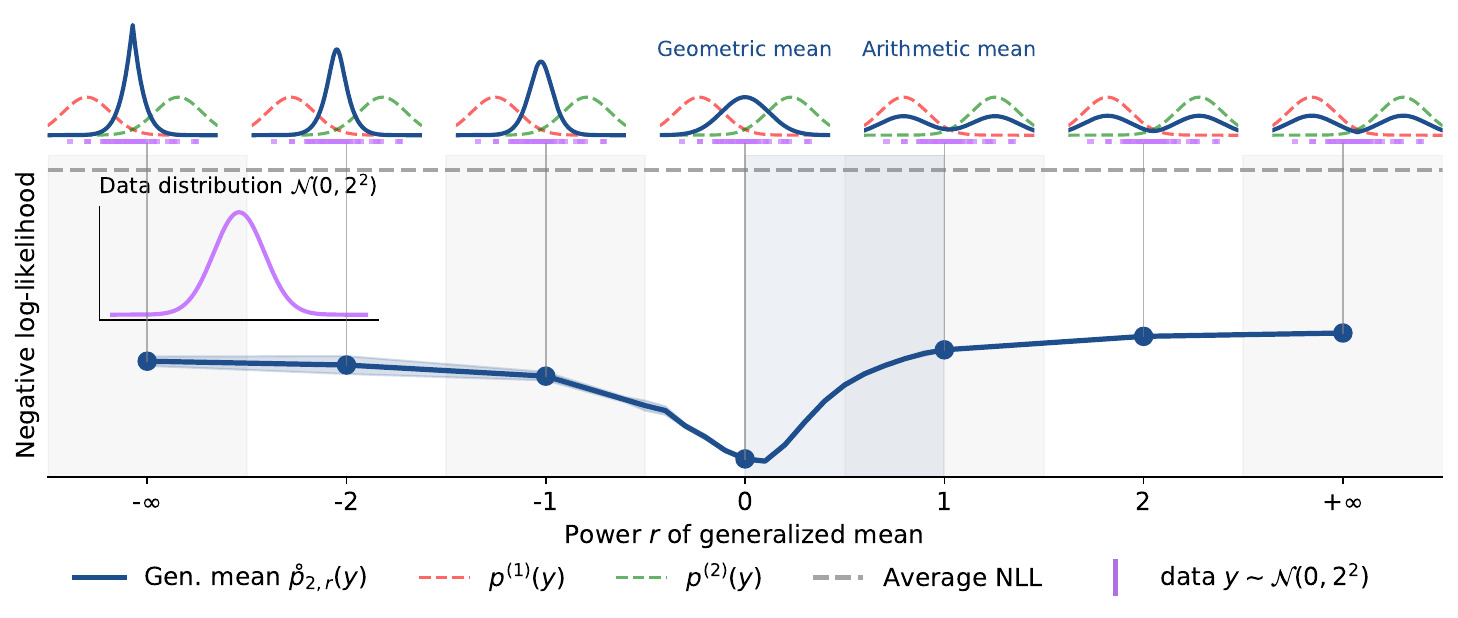}
    
    \vspace{0.5em}
    \includegraphics[width=0.75\linewidth]{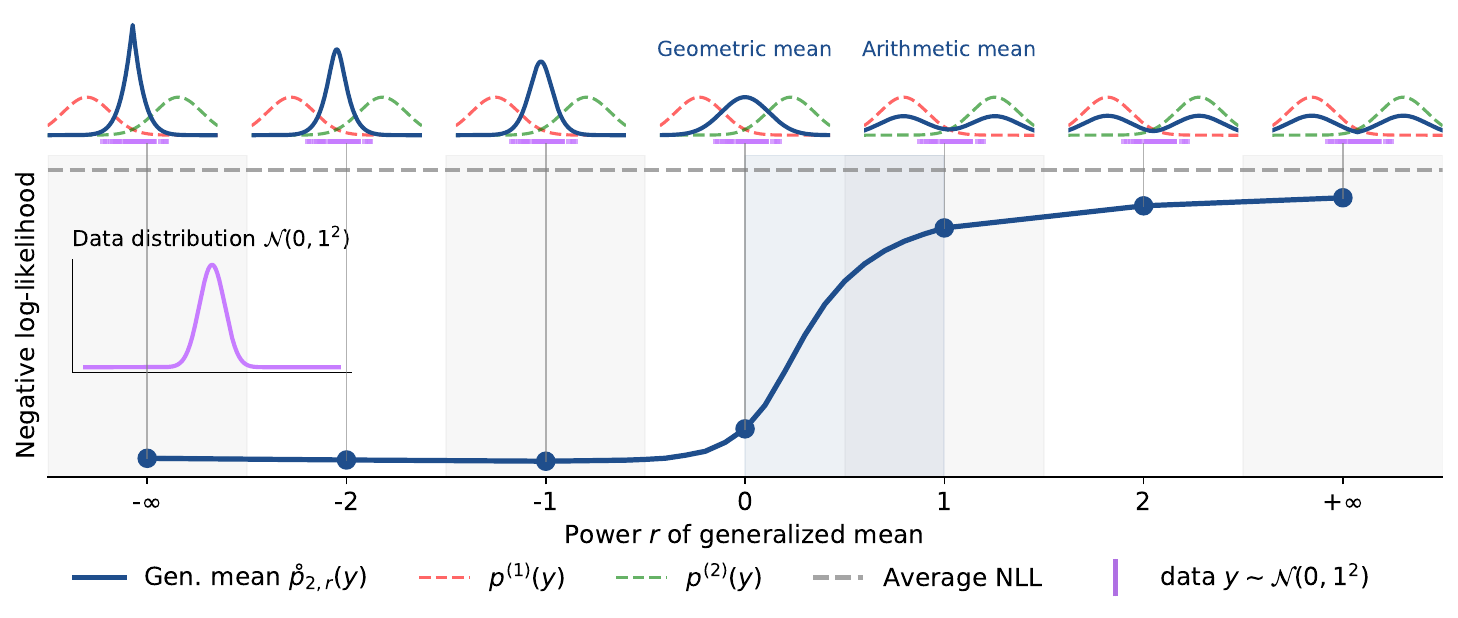}
    
    \vspace{0.5em}
    \includegraphics[width=0.75\linewidth]{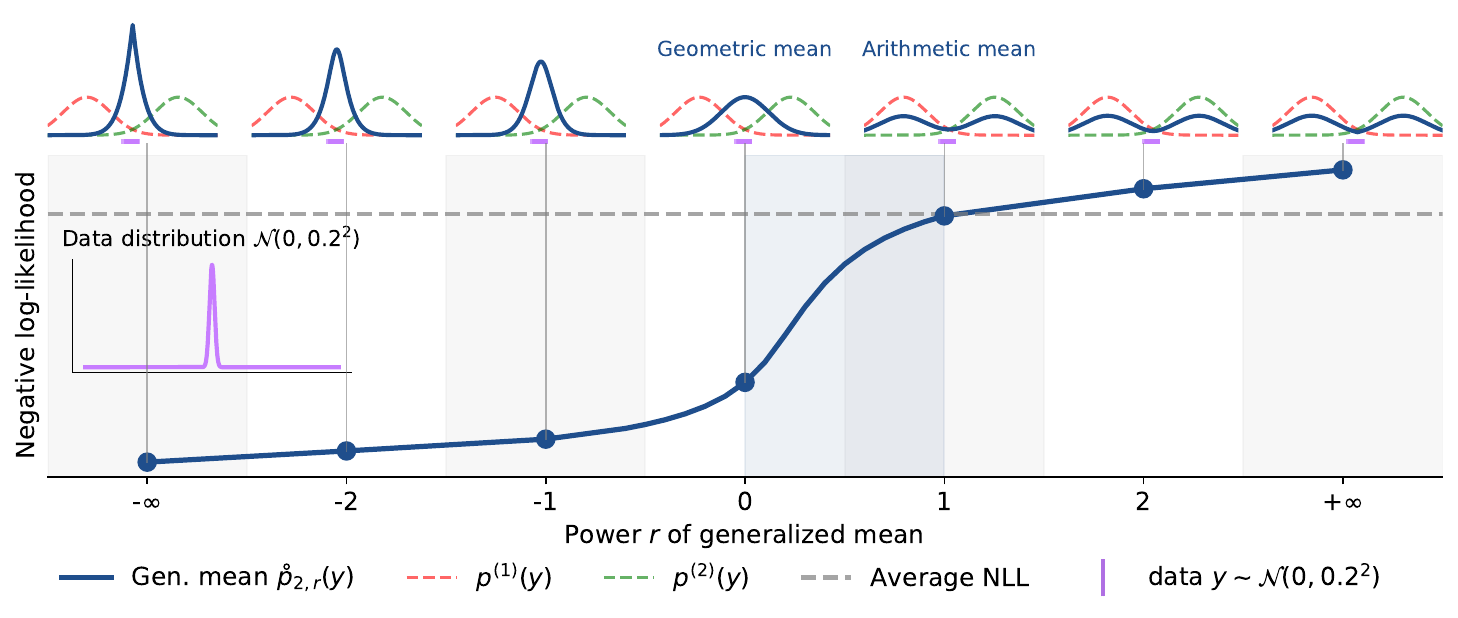}
    
    \caption{
    \textbf{Effect of decreasing $\sigma$ on aggregation behavior.} 
    From top to bottom: large ($\sigma = 3$), moderate ($\sigma = 2$), small ($\sigma = 1$), and very small $\sigma$ ($\sigma = 0.2$). 
    The U-shaped behavior first sharpens, then all $r$ outperform the individual baseline for intermediate $\sigma$, before the trend reverses for very small $\sigma$.
    }
    \label{fig:gaussian_sigma_regimes}
\end{figure}
In this section, we show that concentration of data around agreement points can cause the maximum aggregation to underperform, whereas dispersion of data across the experts’ supports can exacerbate the shortcomings of minimum-type aggregation. 

In \Cref{fig:gaussian_sigma_regimes}, decreasing $\sigma$ (from $\sigma=4$ configuration in \Cref{fig:main figure}) first accentuates the U-shaped behavior. For intermediate values of $\sigma$ ($\{ 3,2,1\}$), all orders $r$ outperform the individual baseline (dashed line). When $\sigma$ becomes very small ($\sigma = 0.2$), the trend reverses: the curve retains a similar structure but with inverted ordering, and minimum-type aggregation underperforms the individual likelihoods. This effect arises because $\sigma$ governs whether samples predominantly fall in disagreement regions (penalizing min-type aggregation) or agreement regions (penalizing max-type aggregation when $p^{(1)} \neq p^{(2)}$), with intermediate values providing a balanced trade-off.

\subsection{How expert separation shapes aggregation performance}

In this section, we show that getting the experts mean closer to each other in comparison to \Cref{fig:main figure} (where $\mu_1 = 3.5$) reduce the effect of power means since $p^{(1)} \approx p^{(2)}$, and putting the experts more separated correct the problem from \Cref{fig:main figure} encountered on negative values of $r$.

In \Cref{fig:expert_separation_regimes}, when $\mu_1$ and $\mu_2$ are close, the effect of the generalized mean is limited and slightly negative for small $r$, as test samples predominantly lie in disagreement regions. Increasing the separation between $\mu_1$ and $\mu_2$ reduces these disagreement regions and concentrates samples around the mode of the normalized generalized mean.
\begin{figure}[t]
    \centering
    
    \includegraphics[width=0.75\linewidth]{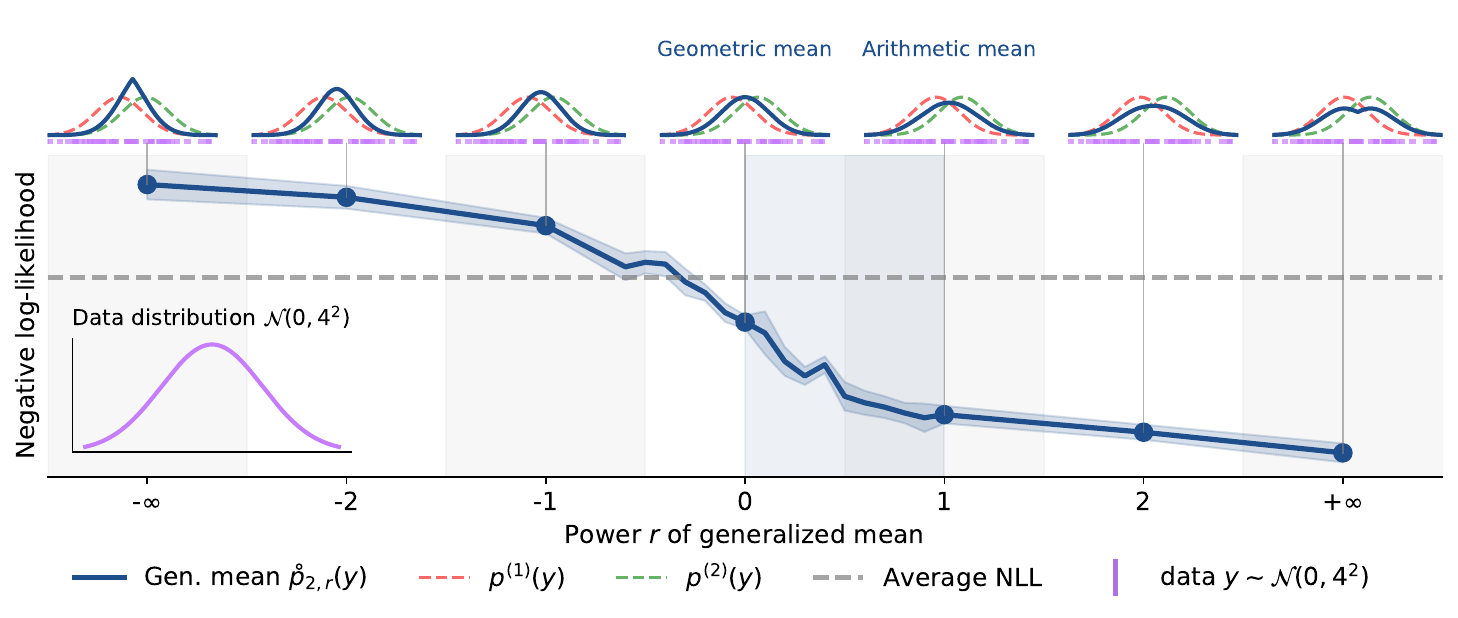}
    
    \vspace{0.6em}
    \includegraphics[width=0.75\linewidth]{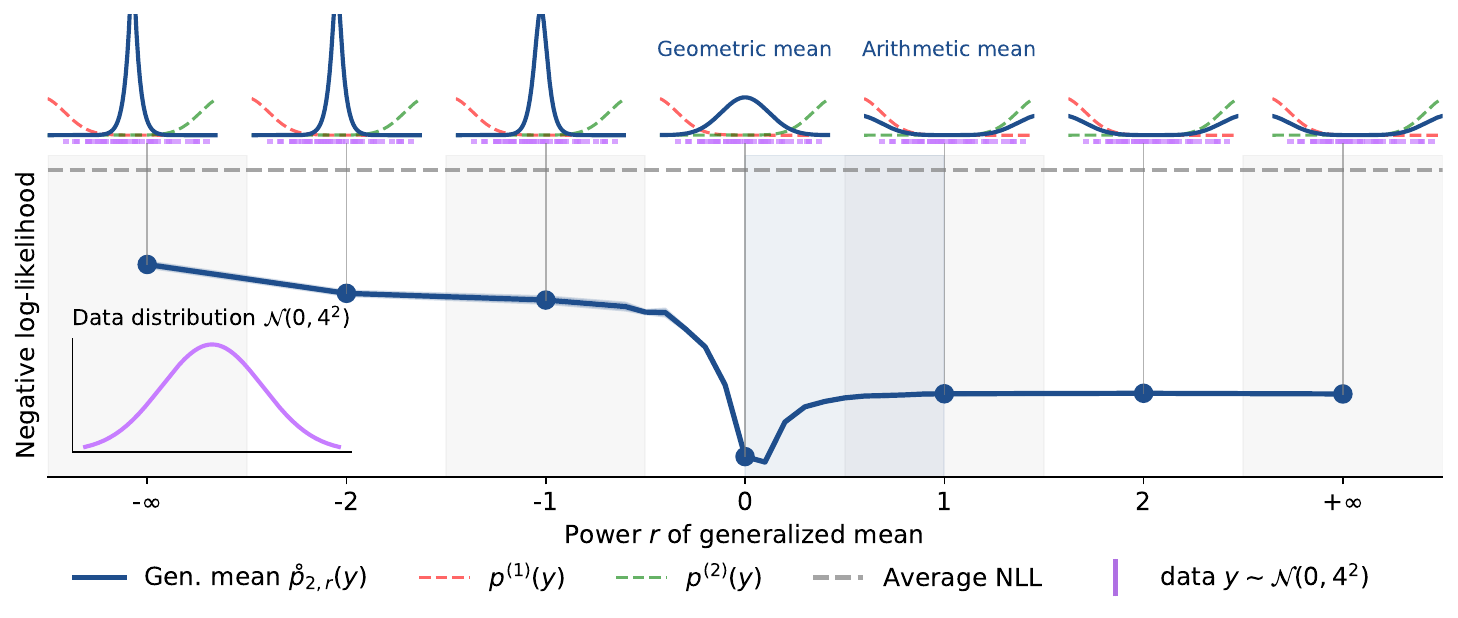}
    
    \vspace{0.6em}
    \includegraphics[width=0.75\linewidth]{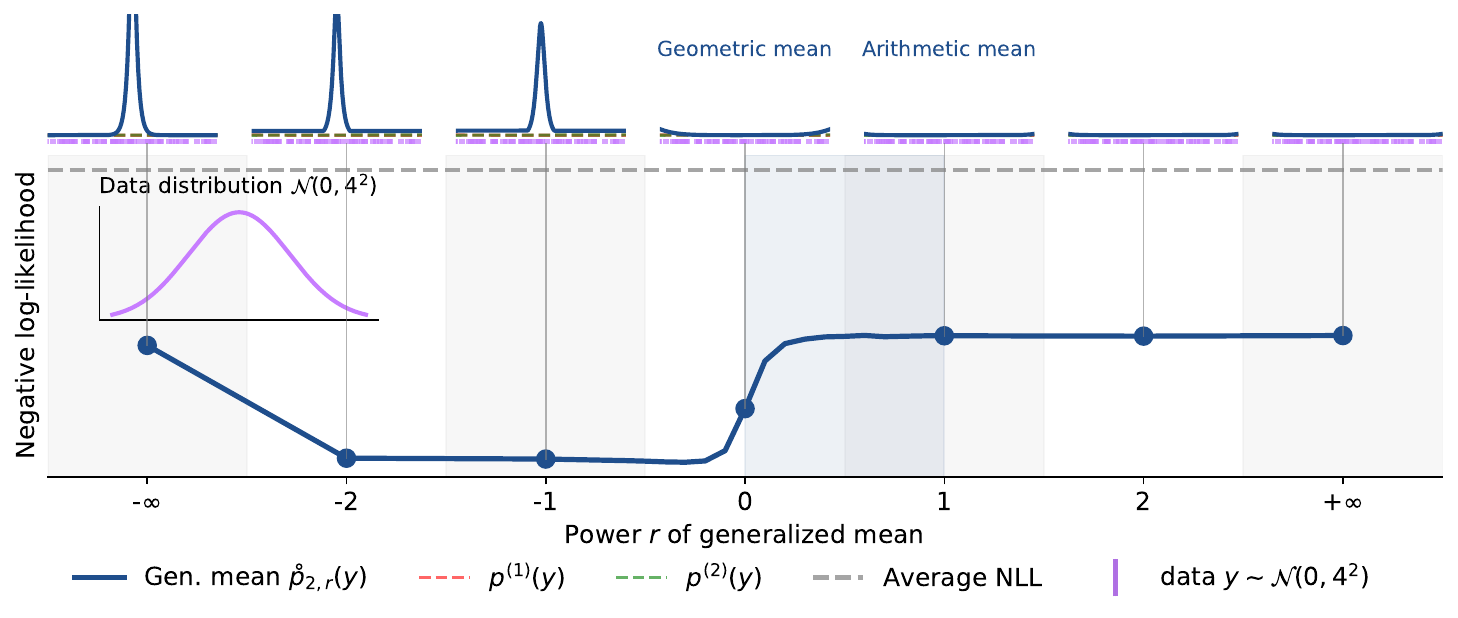}
    
    \caption{
    \textbf{Effect of $\mu_1$ and $\mu_2$ on aggregation performance.}
    From top to bottom: merged experts, moderately separated experts, and highly separated experts. When $\mu_1$ and $\mu_2$ are close, disagreement regions are frequent, which explains the poor performance of min-like aggregations. As the means separate, this effect diminishes as test samples concentrate on the modes of the generalized mean.
    }
    \label{fig:expert_separation_regimes}
\end{figure}

\section{Failure of the ``wisdom of crowds'': Counter-examples}\label{app: proof counterexample}
We provide counter-example in both continuous settings with Gaussian densities, and discrete settings (see \Cref{section: failure cases cex}). While the former examines general regimes outside the reliability interval ($r<0$ and $r>1$), the latter focuses on the extreme aggregation cases corresponding to the minimum and maximum operators.

\subsection{Symmetric Gaussian Densities}\label{app: symmetric gaussians counterexample}
In this section, we prove \Cref{theorem: counterexamples}. We show that we can find two positive densities
$p^{(1)}, p^{(2)} \in \mathcal{P}$ with $p^{(1)} \neq p^{(2)}$
and a point $\vx$, such that
\begin{equation}\label{eq: bad inequality likelihood appendix}
\log \compose{p}_{2,r}(\vx)<\frac{1}{2} \sum_{i=1}^2 \log p^{(i)}(\vx)
\end{equation}
when $r \notin [0,1]$.

\begin{proof}

We consider two symmetric Gaussian distributions. Let $\varphi(x):x \mapsto  (2\pi)^{-1/2} \exp(-x^2/2)$ be the standard normal density. We let $m > 0$ and we consider
\begin{equation}
p^{(1)}(x) = \varphi(x+m), 
\qquad 
p^{(2)}(x) = \varphi(x-m).
\end{equation}
which correspond to $\mathcal{N}(-m,1)$ and $\mathcal{N}(m,1)$ respectively.

To construct counter-examples for $r>1$ and $r<0$, we consider two points of interest: respectively $x=0$, the unique point where the two densities coincide, and $x=m$, which corresponds to the mode of $p^{(2)}$ and lies in the tail of $p^{(1)}$.

\paragraph{The aggregation breaks down at disagreement points when $r < 0$.} We show that if $r<0$ and $x=m$, Jensen's inequality in \Cref{eq: jensen inequality r positive} (and thus \Cref{eq: inequality log power mean}) can be reversed as soon as $m$ is large enough. \Cref{fig:counterexamples_gaussians_m} illustrates how negative $r$ reverses \Cref{eq: inequality log power mean} for $|x|$ and $m$ large enough. We have
\begin{equation}
p^{(2)}(m)=\varphi(0),
\qquad
p^{(1)}(m)=\varphi(2m)=\varphi(0)e^{-2m^2}.
\end{equation}

For any $r<0$, the power mean at $x=m$ satisfies
\begin{align}
M_{2,r}(m)
& =\left(\tfrac12\big(p^{(1)}(m)^r+p^{(2)}(m)^r\big)\right)^{1/r} \\
& =\varphi(0)\left(\tfrac12(1+e^{-2rm^2})\right)^{1/r}.
\end{align}
Taking logarithms and comparing with the average log-density yields
\begin{equation}
\begin{aligned}
\log M_{2,r}(m)
-\tfrac12\!\left(\log p^{(1)}(m)+\log p^{(2)}(m)\right)
\\
=\;
\frac{1}{r}\log\!\Big(\tfrac12(1+e^{-2rm^2})\Big)+m^2.
\end{aligned}
\end{equation}
Since $r<0$, the term $e^{-2rm^2}$ grows exponentially in $m$, thus  $\frac{1}{r}\log\!\Big(\tfrac12(1+e^{-2rm^2}) \Big) + m^2 \sim \frac{1}{r}\log(e^{-2rm^2}) + m^2 = -2m^2 + m^2 = -m^2$ when $m\to+\infty$, so the right-hand side tends to $-\infty$ as $m\to+\infty$.
Therefore, for $m$ large enough,
\begin{equation}
\log M_{2,r}(m)
<
\tfrac12\!\left(\log p^{(1)}(m)+\log p^{(2)}(m)\right),
\end{equation}
which contradicts \Cref{eq: jensen inequality r positive} and \Cref{eq: inequality log power mean}.

\paragraph{The generalized mean fails to improve agreement points when $r>1$.} This time we show that the failure of \Cref{eq: inequality log power mean} on $x = 0$ does not come from reversing \Cref{eq: jensen inequality r positive}, which will in fact be an equality here, but from normalization redistributing probability mass away from this point. \Cref{fig:counterexamples_gaussians_0} illustrates this phenomenon, with a concentrated degradation around this point. By using \Cref{lemma: p-mean inequality}, we have for all $r>1$
\begin{align}\label{eq: Z gaussian inequality}
Z_{2,r}
&= \int_{\R} M_{2,r}(x)\,\mathrm{d}x
> \int_{\R} M_{2,1}(x)\,\mathrm{d}x  \\
&= \int_{\R} \frac{p^{(1)}(x) + p^{(2)}(x)}{2}\,\mathrm{d}x
= 1 .
\end{align}
where \Cref{eq: Z gaussian inequality} holds since the set where $p^{(1)}$ and $p^{(2)}$ differ has strictly positive measure. Now we consider $x=0$, the symmetric point with respect to the two means $\pm m$. We have
\begin{equation}\label{eq: symmetry}
p^{(1)}(0) = p^{(2)}(0) = \varphi(m) > 0.
\end{equation}
Hence, from \Cref{eq: symmetry,eq: Z gaussian inequality} the composed density satisfies
\begin{align}
& \compose{p}_{2,r}(0)
= \frac{M_{2,r}(0)}{Z_{2,r}}
= \frac{\varphi(m)}{Z_{2,r}} \\
\implies &\log \compose{p}_{2,r}(0)
= \log \varphi(m) - \log Z_{2,r}
< \log \varphi(m).
\end{align}
Since 
\begin{equation}
\frac{1}{2} \sum_{i=1}^2 \log p^{(i)}(0) = \frac{1}{2} \sum_{i=1}^2 \log \varphi(m) =  \log \varphi(m),
\end{equation}

it contradicts \Cref{eq: inequality log power mean} and verifies \Cref{eq: bad inequality likelihood appendix}.
\end{proof}

\subsection{Counter-examples in Extreme Aggregation Regimes}
\label{appendix:extreme_counterexamples}

We provide explicit counter-examples in discrete settings  showing
that the likelihood inequality established for
$r \in [0,1]$ does not extend to the extreme aggregation
regimes $r=-\infty$ (minimum aggregation) and
$r=+\infty$ (maximum aggregation). We consider classification, where each predictor outputs a
probability vector over classes and aggregation is performed
componentwise followed by normalization.

\textbf{Failure for $r=-\infty$ at disagreement points. }Consider binary classification with labels
$y \in \{0,1\}$ and true label $y=0$.
Let two classifiers produce
\begin{equation}
p^{(1)} = (0.9,\,0.1),
\qquad
p^{(2)} = (0.01,\,0.99).
\end{equation}

Minimum aggregation selects the smallest probability assigned
to each class:
\begin{equation}
M_{2,-\infty}
=
\big(
\min(0.9,0.01),
\min(0.1,0.99)
\big)
=
(0.01,\,0.1).
\end{equation}

After normalization,
\begin{equation}
\compose{p}_{2,-\infty}
=\frac{(0.01,\,0.1)}{0.01+0.1} = (0.0909,\,0.9091).
\end{equation}

The aggregated log-likelihood of the true label is therefore
\begin{equation}
\log \compose{p}_{2,-\infty}(y=0)
= \log(0.0909)
\approx -2.40.
\end{equation}

The average individual log-likelihood equals
\begin{equation}
\frac{1}{2}
\sum_{i=1}^{2}
\log p^{(i)}(y=0)
=
\frac{1}{2}
\big(
\log 0.9 + \log 0.01
\big)
\approx -2.36.
\end{equation}

Hence,
\begin{equation}
\log \compose{p}_{2,-\infty}(0)
<
\frac{1}{2}
\sum_{i=1}^{2}
\log p^{(i)}(0),
\end{equation}
which contradicts the ``wisdom of crowds'' principle from \Cref{eq: inequality log power mean}.

\textbf{Failure for $r=+\infty$ at agreement points. }
Consider three-class classification with labels
$y \in \{0,1,2\}$ and true label $y=0$.
Let two classifiers produce
\begin{equation}
p^{(1)} = (0.9,\,0.03,\,0.07),
\qquad
p^{(2)} = (0.9,\,0.07,\,0.03).
\end{equation}

Maximum aggregation selects the largest probability assigned
to each class:
\begin{equation}
M_{2,+\infty}
=
\big(
\max(0.9,0.9),
\max(0.03,0.07),
\max(0.07,0.03)
\big)
=
(0.9,\,0.07,\,0.07).
\end{equation}

After normalization,
\begin{equation}
\compose{p}_{2,+\infty}
=
\frac{(0.9,\,0.07,\,0.07)}{0.9+0.07+0.07}
=
(0.865,\,0.067,\,0.067).
\end{equation}

The aggregated log-likelihood of the true label is therefore
\begin{equation}
\log \compose{p}_{2,+\infty}(y=0)
=
\log(0.865)
\approx -0.145.
\end{equation}

The average individual log-likelihood equals
\begin{equation}
\frac{1}{2}
\sum_{i=1}^{2}
\log p^{(i)}(y=0)
=
\frac{1}{2}
\big(
\log 0.9 + \log 0.9
\big)
=
\log 0.9
\approx -0.105.
\end{equation}

Hence,
\begin{equation}
\log \compose{p}_{2,+\infty}(0)
<
\frac{1}{2}
\sum_{i=1}^{2}
\log p^{(i)}(0),
\end{equation}
which again contradicts the log-likelihood inequality from \Cref{eq: inequality log power mean}.

\section{Analytical Expressions of Integrals: Generalized Means of Gaussians}\label{app: analytical expressions}

In this section, we derive analytical expressions for the normalizing constant of the generalized mean for $r \in \mathbb{R}$. We show that closed-form formulas can be obtained within the reliability interval $[0,1]$ (\Cref{section: when provably reliable}), namely for $r=0$ and for values of $r$ which are reciprocals of positive integers. Note that the case $r=1$ is trivial, as the normalizing constant equals one; and for the case $r=0$ the resulting density has been extensively derived \citep{cao2014generalized,cohen2021healing}, notably showing that the weighted geometric mean remains Gaussian. On the other hand, we show that for other values of $r$, including those outside $[0,1]$, the integral do not admit finite closed-form expressions althrough we know that it is finite (\Cref{proposition: generalized p-mean well defined}). This further emphasizes that the interval $[0,1]$ constitutes not only a theoretical reliability regime (\Cref{theorem: log-likelihood inequality}) but also the unique domain where some aggregation retains analytical tractability, whereas moving beyond it provides no comparable analytical or practical advantage.

\paragraph{Motivation in regression.}
In probabilistic regression, Gaussian densities are typically employed as likelihood functions. Studying the generalized mean of Gaussian densities is therefore of practical interest: if the corresponding normalizing constant is analytically tractable, aggregation directly induces a valid Gaussian-based regression model and tractable model likelihood without requiring numerical integration or approximation schemes, except potentially at prediction time where computing the mean still involves an integral.

\paragraph{Notations and goal.}
Let $p^{(1)},\dots,p^{(k)}$ be Gaussian densities in $\mathcal{X} = \R^d$ with $d \geq 1$. Formally, we consider for all $\vx \in \mathcal{X}$
\begin{equation}
p^{(i)}(\vx)
=\frac{1}{(2\pi)^{d/2}\,|\Sigma_i|^{1/2}}
\exp\!\left(
-\frac{1}{2}
(\vx-\mu_i)^\top
\Sigma_i^{-1}
(\vx-\mu_i)
\right),
\qquad i=1,\dots,k,
\end{equation}
where $\mu_i \in \R^d$ denotes the mean vector and $\Sigma_i \in \R^{d\times d}$ is a symmetric positive definite covariance matrix (we write $\Sigma_i  \succ 0$).

The goal is to determine 
\begin{equation}
Z_{k,r} = \int \left( \frac{1}{k} \sum_{i=1}^k p^{(i)}(\vx)^r \right)^{\frac{1}{r}} \vd\vx.
\end{equation}

We distinguish two cases. First, we consider tractable integrals, which correspond to $r=0$ (geometric mean) and $r=1/n$ with $n\in\mathbb{N}\setminus{\{0\}}$, covering a countable subset of the safe interval $[0,1]$. Then, we look at the values $r \in \R$ which do not satisfy the former conditions and do not correspond to tractable integrals.

\subsection{Tractable integrals} \label{app: tractable integrals}

Assume $n \in \N \setminus \{0\}$. The integral of interest is written as
\begin{equation}
Z_{k,r} = \int \left( \frac{1}{k} \sum_{i=1}^k p^{(i)}(\vx)^{\frac{1}{n}} \right)^{n} \vd\vx.
\end{equation}

This formulation allows us to directly apply the following multinomial theorem, defined as follows.
\begin{lemma}[\textbf{Multinomial theorem}]\label{lemma: multinomial}
Let $k$ be a positive integer and $n$ be a non-negative one. 
Let $a_1,\dots,a_k \in F$ where $F$ is a field (e.g., $F=\R$). Then,
\begin{equation}\label{eq: multinomial}
(a_1 + \cdots + a_k)^n
=
\sum_{\substack{n_1+\cdots+n_k=n \\ n_i \ge 0}}
\binom{n}{n_1,\dots,n_k}
\prod_{i=1}^{k} a_i^{\,n_i},
\end{equation}
where the multinomial coefficient is defined as
\begin{equation}
\binom{n}{n_1,\dots,n_k}
=
\frac{n!}{n_1!\cdots n_k!}.
\end{equation}
\end{lemma}
Mainly from \Cref{eq: multinomial}, we can derive the proposition below.
\begin{proposition}\label{proposition: z for 1 over n} Let $p^{(1)},\dots,p^{(k)}$ be $k \geq 1$ Gaussian densities of parameters $(\mu_i,\Sigma_i)$. If $r = \frac{1}{n}$ with $n \in \N \setminus \{0\}$, we have
\begin{equation}
\begin{aligned}
Z_{k,r}
&= \frac{1}{k^n}\sum_{\substack{n_1+\cdots+n_k=n \\ n_i \ge 0}} 
\frac{n!}{n_1!\cdots n_k!}
(2 \pi)^{\frac{d}{2}}
\Big| \sum_{i=1}^k \frac{n_i}{n} \Sigma_i^{-1} \Big|^{-\frac{1}{2}}
\prod_{i=1}^k (2\pi)^{-\frac{d n_i}{2n}}
| \Sigma_i |^{-\frac{n_i}{2n}} \\
&\quad \cdot
\exp\Bigg(
\frac{1}{2}
\Big(\sum_{i=1}^k \frac{n_i}{n} \mu_i^\top \Sigma_i^{-1}\Big)
\Big(\sum_{i=1}^k \frac{n_i}{n} \Sigma_i^{-1}\Big)
\Big(\sum_{i=1}^k \frac{n_i}{n} \Sigma_i^{-1} \mu_i\Big)
-\frac{1}{2}
\sum_{i=1}^k \frac{n_i}{n} \mu_i^\top \Sigma_i^{-1} \mu_i
\Bigg).
\end{aligned}
\end{equation}
If $r=0$, 
\begin{equation}
\begin{aligned}
Z_{k,r} = & (2 \pi)^{\frac{d}{2}}
\Big| \sum_{i=1}^k \frac{1}{k} \Sigma_i^{-1} \Big|^{-\frac{1}{2}}
\prod_{i=1}^k (2\pi)^{-\frac{d}{2k}}
| \Sigma_i |^{-\frac{1}{2k}} \\
&\quad \cdot \exp\Bigg(
\frac{1}{2}
\Big(\sum_{i=1}^k \frac{1}{k} \mu_i^\top \Sigma_i^{-1}\Big)
\Big(\sum_{i=1}^k \frac{1}{k} \Sigma_i^{-1}\Big)
\Big(\sum_{i=1}^k \frac{1}{k} \Sigma_i^{-1} \mu_i\Big)
-\frac{1}{2}
\sum_{i=1}^k \frac{1}{k} \mu_i^\top \Sigma_i^{-1} \mu_i
\Bigg).
\end{aligned}
\end{equation}

\end{proposition}

\begin{proof}
From \Cref{lemma: multinomial}, we have
\begin{equation}\label{eq: from lemma d}
Z_{k,r} = \frac{1}{k^n}\sum_{\substack{n_1+\cdots+n_k=n \\ n_i \ge 0}} 
\binom{n}{n_1,\dots,n_k}
\underbrace{\int \prod_{i=1}^k p_i(\vx)^{\frac{n_i}{n}}}_{I_k}.
\end{equation}
We remark that $I_k$ is the integral of a weighted geometric mean since $\sum_{i=1}^k \frac{n_i}{n} = 1$. Let us denote the weights as $w_i = \frac{n_i}{n}$ for all $i$. For all 
$\vx \in \mathcal{X}$, we have
\begin{align}
\prod_{i=1}^k p_i(\vx)^{w_i} & = \prod_{i=1}^k (2\pi)^{-\frac{dw_i}{2}}| \Sigma_i |^{-\frac{w_i}{2}} \exp \left(-\frac{1}{2} \sum_{i=1}^k w_i (\vx - \mu_i)^\top \Sigma_i^{-1} (\vx - \mu_i)\right) \label{eq: weighted geometric mean first expression} \\
& = C_k\exp \left( -\frac{1}{2} \big( \vx^\top Q_k \vx - 2 b_k^\top \vx + c_k \big)\right)
\end{align}
where
\begin{equation}\label{eq: intermediate quantities}
C_k = \prod_{i=1}^k (2\pi)^{-\frac{dw_i}{2}}| \Sigma_i |^{-\frac{w_i}{2}}, \quad Q_k = \sum_{i=1}^k w_i \Sigma_i^{-1} \succ 0, \quad b_k = \sum_{i=1}^k w_i \Sigma_i^{-1} \mu_i, \quad c_k = \sum_{i=1}^k w_i \mu_i^\top \Sigma_i^{-1} \mu_i.
\end{equation}
Completing the square allows us to factor out a term with an explicit integral. Indeed,
\begin{align}
\prod_{i=1}^k p_i(x)^{w_i} & =  C_k \exp \left( -\frac{1}{2} \big( \vx^\top Q_k \vx - 2 b_k^\top \vx \big) \right) \exp\left( -\frac{1}{2} c_k \right)\\
& =  C_k \exp \left( -\frac{1}{2} \big( \vx^\top Q_k \vx - 2 (Q_k^{-1} b_k )^\top Q_k \vx \big) \right) \exp\left( -\frac{1}{2} c_k \right) \label{eq: Q symmetric} \\
& =  C_k \exp \left( -\frac{1}{2} \big( \vx^\top Q_k \vx - 2 (Q_k^{-1} b_k )^\top Q_k \vx \big) + \underbrace{(Q_k^{-1} b_k)^\top Q_k (Q_k^{-1} b_k)}_{b_k^\top Q_k b_k} \right) \exp\left( \frac{1}{2} b_k^\top Q_k b_k -\frac{1}{2} c_k \right)\\& = C_k \exp \left( -\frac{1}{2} \big( \vx - Q_k^{-1}b_k)^\top Q_k (\vx - Q_k^{-1}b_k)\big)\right) \exp\left( \frac{1}{2} b_k^\top Q_k b_k -\frac{1}{2} c_k \right)\\
& = C_k (2 \pi)^{\frac{d}{2}} | Q_k |^{-\frac{1}{2}} \,\, \mathcal{N}(\vx; \, Q_k^{-1}b_k, Q_k^{-1}) \, \exp\left( \frac{1}{2} b_k^\top Q_k b_k -\frac{1}{2} c_k \right)
\end{align}
where \Cref{eq: Q symmetric} follows from the fact that $Q_k$ is invertible and symmetric. Using $\int\mathcal{N}(\vx; \, Q_k^{-1}b_k, Q_k^{-1}) \vd \vx =1 $, we have
\begin{equation}\label{eq: integral of product}
\int \prod_{i=1}^k p_i(\vx)^{w_i} \vd \vx = (2 \pi)^{\frac{d}{2}} | Q_k |^{-\frac{1}{2}}\prod_{i=1}^k (2\pi)^{-\frac{dw_i}{2}}| \Sigma_i |^{-\frac{w_i}{2}} \exp\left( \frac{1}{2} b_k^\top Q_k b_k -\frac{1}{2} c_k \right)
\end{equation}
where $Q_k,b_k,c_k$ are defined in \Cref{eq: intermediate quantities}. We finally use \Cref{eq: from lemma d} to prove the first result of \Cref{proposition: z for 1 over n}.

For the case $r=0$, we just apply the same reasoning from \Cref{eq: weighted geometric mean first expression} leading to \Cref{eq: integral of product} with $w_i = \frac{1}{k}$.
\end{proof}

\subsection{Non trivial integrals}\label{not tractable integrals}

As soon as $r$ does not belong to the cases studied in \Cref{app: tractable integrals}, and in particular when $r \notin [0,1]$, we cannot derive explicit integrals $Z_{k,r}$ with a finite sum like before. 

To see this, we consider the easy case $k=2$, and assume $\alpha$ is any real (or complex) number. It can be negative. Without restriction, \Cref{lemma: multinomial} is generalized through Binomial series, as stated in the following lemma

\begin{lemma}[\textbf{Binomial Series}]\label{lemma: binomial series}
Let $a_1,a_2 \in \R$. Then, if $|a| < |b|$,
\begin{equation}
(a+b)^\alpha = \sum_{i = 0}^\infty \binom{\alpha}{i} a^k b^{\alpha-k}. 
\end{equation}
where the generalized binomial coefficient is defined by
\begin{equation}
\binom{\alpha}{i} = \frac{\alpha (\alpha-1)(\alpha-2)\dots (\alpha-i-1)}{i!}
\end{equation}
\end{lemma}

From this lemma, $Z_{k,r}$ can be expressed for general 
$r \in \mathbb{R} \setminus \{0\}$ as an integral involving a binomial series, 
following the same kind of approach as in \Cref{app: tractable integrals}. 
To ensure convergence of the series expansion, the integral must be decomposed 
into two regions corresponding to the cases $p^{(1)} < p^{(2)}$ and 
$p^{(2)} < p^{(1)}$. However, unless $r^{-1}$ is a positive integer (as in \Cref{app: tractable integrals}), 
the resulting binomial expansions are no longer finite sums. Consequently, 
evaluating $Z_{k,r}$ generally requires approximation or numerical procedures, 
either for computing the integral itself or for truncating the associated 
series representation.

\end{document}